\documentclass{article}

\usepackage[preprint,nonatbib]{neurips_2024}

\usepackage[utf8]{inputenc}
\usepackage[T1]{fontenc}
\usepackage{hyperref}
\usepackage{url}
\usepackage{booktabs}
\usepackage{amsfonts}
\usepackage{amsmath}
\usepackage{amssymb}
\usepackage{amsthm}
\usepackage{nicefrac}
\usepackage{microtype}
\usepackage{graphicx}
\usepackage{xcolor}
\usepackage{algorithm}
\usepackage{algorithmic}
\usepackage{subcaption}
\usepackage{multirow}
\usepackage{tabularx}
\usepackage{enumitem}

\newtheorem{theorem}{Theorem}
\newtheorem{lemma}[theorem]{Lemma}
\newtheorem{corollary}[theorem]{Corollary}
\newtheorem{proposition}[theorem]{Proposition}
\newtheorem{remark}{Remark}
\newtheorem{definition}{Definition}

\title{Quantization Error Is Spectrally Flat: A Single Random Probe\\Is a Calibrated, Data-Free Sensitivity Estimator, with Application to\\Budget-Targeted Mixed-Precision Quantization at 400B Scale}

\author{
  I.~Kennedy \qquad T.~Kennedy \\
  (http://baa.ai Blacksheep ai.) \\
  (Auckland, New Zealand)
}

\begin{document}

\maketitle

\begin{abstract}
A single random Gaussian probe gives an unbiased estimate of the squared Frobenius norm of a layer's quantization error (Theorem~\ref{thm:main}). This estimator is far better behaved than Monte-Carlo intuition suggests, because of a structural fact about quantization error that we measure in closed form: \emph{the error spectrum is flat}. Across 1{,}683 weight tensors of a 35B-parameter MoE and a 9B dense model, the effective dimensionality $d_{\mathrm{eff}} = \mathrm{Tr}(\Delta^\top\Delta)^2/\|\Delta^\top\Delta\|_F^2$ of the round-to-nearest error $\Delta$ is $0.93$ to $0.96$ times the value $mn/(m{+}n)$ attained by an i.i.d.\ noise matrix of the same shape, and on the MoE model its median is the same to three digits from 2-bit to 8-bit. Flatness makes the probe estimator \emph{calibrated}: its coefficient of variation is $\sqrt{2/d_{\mathrm{eff}}}$, predictable from tensor shape before touching the weights, and it matches measurement to $0.1\%$ at the median. One probe measures per-tensor sensitivity to within $4$ to $7\%$, twenty probes to within $1.3$ to $1.4\%$, and averaging converges at the i.i.d.\ rate against exact ground truth.

\textbf{RAM} (\textbf{R}esource-\textbf{A}ware \textbf{M}ixed-precision) applies the propagated form of the same estimator to budget-targeted quantization. Gaussian probes carrying the network's own input statistics score every tensor at six candidate bit-widths, which by a one-line extension of the theorem estimates the layer-wise objective that GPTQ minimizes with calibration data; a multiple-choice knapsack solver allocates bit-widths under an exact byte budget, and divergence guardrails block catastrophic 2-bit assignments. One probe pass serves any budget. The isolated and propagated signals are different quantities: they rank tensors independently on Qwen3.5-35B-A3B (Spearman $-0.01$), and calibration transfers only partly (per-sequence CV 0.16 to 2.4 against 0.04 to 0.07), yet rank agreement with an independent reference at the allocator's 50 sequences is 0.97 to 1.00 because the propagated signal spans ten times the range across tensors. Propagation supplies the missing input statistics: on Qwen3-8B, Qwen3.5-9B and Qwen3.8-27B the propagated probe's estimate of the layer-wise calibration objective rank-correlates 0.83, 0.81 and 0.83 with the objective computed from real activations (domain ceiling 0.99), while the isolated estimator is uncorrelated with it ($-0.07$ to $-0.15$). That objective is nevertheless the wrong quantity to allocate from: at matched bytes on Qwen3.8-27B, allocating llama.cpp k-quant types from the block-output probe signal beats the vendor IQ3\_M mix by 5.5\% perplexity and beats an oracle that allocates from the real-activation objective by 1.6\%, because the block signal carries the downstream weighting the per-tensor objective lacks. Fed through one allocator at matched bytes on Qwen3-8B, the propagated probe ties calibration-based HAWQ-V2 Hessian-trace sensitivity on WikiText-2 (paired $p = 0.96$), both beat the isolated estimator by $0.6\%$ ($p < 10^{-15}$), and all three tie on MMLU; on Qwen3.5-35B-A3B the isolated and propagated allocations tie on perplexity and the isolated one leads MMLU by $1.1$ points. The probe pass takes 36 to 539 seconds on a single Apple Silicon machine for models from 8B to 400B parameters (15 to 803\,GB in BF16). On the four MoE models where uniform 4-bit builds of comparable size exist, RAM builds reach 3.5 to 13.6\% lower median WikiText-2 perplexity. On Qwen3.5-35B-A3B, RAM beats a GPTQ-style ``protect attention'' allocation by 1.7\% perplexity at 4.8\% smaller size, and at 33.6\,GB it scores 73.1\% on MMLU against 72.1\% for the BF16 model. Latent probes that pass through the compression path recover the sensitivity of post-bottleneck weights in Multi-head Latent Attention models, where isotropic probes fail. Pre-quantized models are published at \url{https://huggingface.co/baa-ai}.
\end{abstract}

% ==============================================================================
\section{Introduction}
\label{sec:intro}
% ==============================================================================

Post-training quantization (PTQ) is the main route to running large language models (LLMs) on consumer hardware. GPTQ~\cite{frantar2023gptq}, AWQ~\cite{lin2024awq}, and SqueezeLLM~\cite{kim2024squeezellm} compress well, and all three need a representative calibration dataset. Calibration data may be unavailable for proprietary models, the chosen distribution may not match the deployment domain, and calibration itself costs compute: on a 400B-parameter Mixture-of-Experts (MoE) model, GPTQ-style calibration loads the full model and runs forward passes over hundreds of sequences, which takes hours on high-end hardware.

Data-free approaches~\cite{tang2023easyquant,zhang2025mxq,badri2024hqq} avoid calibration but either apply uniform bit-widths or rank tensors by a single weight statistic (kurtosis~\cite{akhondzadeh2025kurtail}, the Frobenius norm of the rounding error~\cite{zhang2025mxq}) with hand-tuned thresholds. Threshold-based allocation produces one output size regardless of the deployment memory target.

We take a different route. Instead of summarizing weight matrices with statistical features, we measure how quantization error propagates through each layer using random functional probes:

\begin{quote}
\emph{A random Gaussian vector passed through a linear layer provides an unbiased estimate of the squared Frobenius norm of the quantization error, the same quantity that calibration-based methods estimate using real activations.}
\end{quote}

This observation, formalized in Theorem~\ref{thm:main}, says that a single random probe estimates in expectation the quantity that calibration-based methods estimate with data. Unbiasedness alone would be useless if the estimator's variance were large. The second half of the story is a structural property of quantization error that we measure rather than assume: \emph{round-to-nearest quantization error is spectrally flat}. Its effective dimensionality $d_{\mathrm{eff}}(\Delta^\top\Delta)$ is within a few percent of the value an i.i.d.\ noise matrix of the same shape would have ($mn/(m{+}n)$, the Marchenko--Pastur prediction), which is hundreds to thousands of significant directions for transformer-sized tensors. Because the single-probe coefficient of variation is $\sqrt{2/d_{\mathrm{eff}}}$ (Lemma~\ref{lem:variance}), flatness places the estimator in its best-case concentration regime: one probe yields a $4$ to $7\%$ measurement of per-tensor sensitivity and twenty probes yield $1.3$ to $1.4\%$.

Building on this, \textbf{RAM} (\textbf{R}esource-\textbf{A}ware \textbf{M}ixed-precision) is a complete data-free mixed-precision quantization framework. This paper makes the following contributions:
\begin{enumerate}[leftmargin=*,topsep=2pt,itemsep=1pt]
    \item \textbf{Theoretical guarantee}: We prove that $\mathbb{E}[\mathbf{x}^\top \Delta^\top \Delta\, \mathbf{x}] = \|\Delta\|_F^2$ for Gaussian $\mathbf{x}$, where $\Delta = \mathbf{W} - \hat{\mathbf{W}}$ is the quantization error matrix (Theorem~\ref{thm:main}, \S\ref{sec:probes}).
    \item \textbf{Measured spectral flatness}: We compute the effective dimensionality of quantization error in closed form on 1{,}683 tensors across two architectures and four bit-widths: $d_{\mathrm{eff}} = 0.93$ to $0.96 \times mn/(m{+}n)$, the i.i.d.-noise value. The resulting per-probe accuracy $\sqrt{2/d_{\mathrm{eff}}}$ is predictable from tensor shape alone and verified to $0.1\%$ median accuracy (\S\ref{sec:saturation}).
    \item \textbf{Honest convergence}: Probe averaging converges at the i.i.d.\ rate ($p^{-0.50}$ measured against exact closed-form ground truth with disjoint probe blocks); an earlier version of this work reported a faster rate that was an artifact of a nested probe reference (\S\ref{sec:saturation_experiment}).
    \item \textbf{Propagated functional probes}: Random Gaussian probes propagated through the network score every tensor at six candidate bit-widths by output-space divergence, in 36 seconds to 9 minutes per model with no calibration data (\S\ref{sec:signal}). The propagated probe is the same estimator applied to the input- and downstream-weighted error operator (Proposition~\ref{prop:weighted}), which with the network's own input statistics is the calibration objective of GPTQ and OBQ without data.
    \item \textbf{Isolated and propagated sensitivities are different quantities}: the two signals rank tensors independently on Qwen3.5-35B-A3B (Spearman $-0.01$ over 390 tensors) and calibration transfers only partly: per-sequence CV 0.16 to 2.4 against 0.04 to 0.07, yet rank agreement at the allocator's 50 sequences is 0.97 to 1.00 because the propagated signal spans ten times the range across tensors (\S\ref{sec:propagated_calibration}, \S\ref{sec:iso_vs_prop}).
    \item \textbf{Data-free estimate of the calibration objective}: the propagated probe's quadratic form rank-correlates 0.83 (Qwen3-8B), 0.81 (Qwen3.5-9B) and 0.83 (Qwen3.8-27B, with the pipeline's adaptation rule switched off) with the GPTQ layer objective under real activations, against $-0.07$ to $-0.15$ for the isolated estimator and 0.25 for HAWQ-V2; the block-output cosine the pipeline uses reaches only 0.50 and 0.11. Allocating from the objective itself, with real data or without, nevertheless loses to the block-output signal at matched bytes on Qwen3.8-27B (\S\ref{sec:objective}).
    \item \textbf{Matched-budget ablations}: through one allocator at one budget, the isolated and propagated allocations tie on Qwen3.5-35B-A3B on perplexity, with the isolated one $1.1$ MMLU points ahead; on Qwen3-8B the propagated probe ties calibration-based HAWQ-V2 (paired $p = 0.96$), both beat the isolated estimator by $0.6\%$ perplexity ($p < 10^{-15}$), and all three tie on MMLU (\S\ref{sec:iso_ablation}, \S\ref{sec:hawqv2}).
    \item \textbf{Budget-targeted allocation}: Per-tensor scores feed a multiple-choice knapsack solver under an explicit byte budget, with divergence guardrails against catastrophic 2-bit assignments; one probe pass serves any budget (\S\ref{sec:allocation}).
    \item \textbf{Seven-architecture evaluation}: Quality results on models from 8B to 122B parameters (15 to 250\,GB in BF16), and probe timing up to 400B parameters (803\,GB), covering dense, MoE with up to 256 experts, FP8-native, and MLA architectures (\S\ref{sec:experiments}).
    \item \textbf{Blanket ``protect attention'' rules lose to measured sensitivity}: A GPTQ-style allocation that leaves all attention and shared-expert tensors unquantized is both larger and worse than the probe allocation; the probes rank attention output/value projections and MoE routers most sensitive and routed experts least, so bytes go to specific tensors rather than to a tensor class (\S\ref{sec:gptq_comparison}, \S\ref{sec:sensitivity_patterns}).
    \item \textbf{MLA probe coverage}: We identify and resolve a limitation of functional probes on Multi-head Latent Attention (MLA)~\cite{deepseekai2024deepseekv2} architectures, where compressed KV projections create sensitivity blind spots (\S\ref{sec:mla}).
\end{enumerate}

% ==============================================================================
\section{Related Work}
\label{sec:related}
% ==============================================================================

\paragraph{Calibration-based PTQ.}
GPTQ~\cite{frantar2023gptq} applies layerwise Optimal Brain Quantization using Hessian information from calibration data, processing one column at a time with error compensation. AWQ~\cite{lin2024awq} identifies salient weight channels via activation-aware scaling. SqueezeLLM~\cite{kim2024squeezellm} combines sensitivity-based non-uniform quantization with dense-and-sparse decomposition. SpQR~\cite{dettmers2024spqr} isolates outlier weights for full-precision storage. QuIP~\cite{chee2024quip} and AQLM~\cite{egiazarian2024aqlm} apply random orthogonal transformations and additive codebooks respectively. All require representative calibration data.

\paragraph{Data-free quantization.}
EasyQuant~\cite{tang2023easyquant} proposes data-free PTQ via weight distribution analysis. MXQ~\cite{zhang2025mxq} uses the Frobenius norm of quantization error as a single sensitivity metric for mixed-precision allocation. HQQ~\cite{badri2024hqq} uses half-quadratic splitting for fast uniform quantization. HIGGS~\cite{malinovskii2024higgs} applies Hadamard rotation via the linearity theorem. Lee et al.~\cite{lee2021datafreemp} proposed a data-free mixed-precision sensitivity metric for vision networks with synthetic labelled data. Concurrent calibration-free allocators for LLMs rank layers or experts by weight-spectrum statistics: NSDS~\cite{zhang2026nsds} combines numerical and structural layer sensitivities, AlphaQ~\cite{yang2026alphaq} scores experts with a heavy-tailed self-regularization exponent, and BitsMoE~\cite{zhao2026bitsmoe} allocates by spectral energy of an SVD factorization. RAM differs in deriving its signal from the output-space effect of quantizing each tensor under random inputs, and in reporting a measured calibration guarantee for that signal.

\paragraph{Sensitivity-based mixed-precision.}
HAWQ~\cite{dong2019hawq} and HAWQ-V2~\cite{dong2020hawqv2} allocate bit-widths by per-tensor Hessian-trace sensitivity, with the trace estimated by Hutchinson's method on calibration data; \S\ref{sec:hawqv2} compares RAM against HAWQ-V2 at matched bytes. LLM-MQ~\cite{li2023llmmq} formulates mixed-precision allocation as an integer program with calibration-based sensitivities. SliM-LLM~\cite{huang2024slimlm} achieves group-wise mixed precision via salience-driven allocation, and CherryQ~\cite{cui2024cherryq} identifies high-impact parameters; both use calibration data. MC-MoE~\cite{huang2024mcmoe} combines expert quantization with dynamic pruning for MoE models, and Q-Strata~\cite{lee2026qstrata} allocates across MoE blocks with a model-level objective evaluated on the assembled quantized model. GAMMA~\cite{yao2026gamma} and MixQuant~\cite{misra2026mixquant} share RAM's goal of one offline analysis that serves any budget; GAMMA learns module preferences with a hidden-state reconstruction objective, and MixQuant marginalizes each layer's distortion over random upstream quantization configurations. CASA~\cite{yoshida2026casa} shows that scalar sensitivity proxies in the multiple-choice knapsack formulation can distort inter-module rankings when the input and output Hessian factors are ill-conditioned; RAM's propagated probes address the input side by using the network's own upstream statistics.

\paragraph{Quantization-error structure.}
LQER~\cite{zhang2024lqer} applies an activation-induced scale matrix to drive the singular-value distribution of the quantization error toward a distribution that a low-rank correction can absorb. Our measurement concerns the unscaled round-to-nearest error and finds the opposite regime: its spectrum is flat, so no low-rank summary captures it, and that flatness is what makes random probes accurate.

\paragraph{Stochastic trace estimation.}
Hutchinson's estimator~\cite{hutchinson1989stochastic} approximates $\mathrm{Tr}(A)$ via $\mathbb{E}[\mathbf{z}^\top A \mathbf{z}]$ with Gaussian or Rademacher probes; Bekas et al.~\cite{bekas2007estimator} extend it to the diagonal, and Avron and Toledo~\cite{avron2011randomized} give concentration bounds. Theorem~\ref{thm:main} applies this machinery with $A = \Delta^\top\Delta$, where $\Delta$ is the quantization residual. The Johnson--Lindenstrauss lemma~\cite{johnson1984extensions} gives the same conclusion from the projection side: random projections preserve norm structure, so probe-based rankings track the true sensitivity ordering. Fisher information~\cite{amari1998natural} and Hessian-based sensitivity~\cite{frantar2023gptq} are the classical importance measures; both need data or activations.

% ==============================================================================
\section{Method}
\label{sec:method}
% ==============================================================================

RAM operates in three stages: (1) score every tensor at every candidate bit-width with functional probes, (2) apply divergence guardrails, and (3) solve the budget-constrained bit-width allocation. The pipeline needs only the pretrained weights: no calibration data, no gradients, no forward passes over text. Section~\ref{sec:probes} analyses the probe estimator on an isolated tensor, where its statistics can be computed in closed form; \S\ref{sec:signal} describes the propagated version the allocator uses.

\subsection{Functional Probes on an Isolated Tensor}
\label{sec:probes}

Consider a linear layer $f(\mathbf{x}) = \mathbf{W}\mathbf{x}$ with weight matrix $\mathbf{W} \in \mathbb{R}^{m \times n}$. Let $\hat{\mathbf{W}}$ denote the quantized weight and $\Delta = \mathbf{W} - \hat{\mathbf{W}}$ the quantization error matrix. The output perturbation for input $\mathbf{x}$ is $\delta(\mathbf{x}) = \Delta\,\mathbf{x}$, with squared norm $\|\delta(\mathbf{x})\|^2 = \mathbf{x}^\top \Delta^\top \Delta\, \mathbf{x}$. Rather than estimating this quantity with calibration data, we draw \emph{functional probes} $\mathbf{x} \sim \mathcal{N}(\mathbf{0}, \mathbf{I}_n)$ and compute the probe sensitivity
\begin{equation}
    S_{\text{probe}}(\mathbf{W}, \hat{\mathbf{W}}) = \frac{1}{p}\sum_{j=1}^{p} \|\Delta\,\mathbf{x}_j\|^2, \qquad \mathbf{x}_j \sim \mathcal{N}(\mathbf{0}, \mathbf{I}_n),
    \label{eq:probe_sensitivity}
\end{equation}
where $p$ is the number of probes.

\begin{theorem}[Unbiased probe sensitivity]
\label{thm:main}
Let $\mathbf{x} \sim \mathcal{N}(\mathbf{0}, \mathbf{I}_n)$ and $\Delta \in \mathbb{R}^{m \times n}$. Then
\begin{equation}
    \mathbb{E}\!\left[\mathbf{x}^\top \Delta^\top \Delta\, \mathbf{x}\right] = \mathrm{Tr}(\Delta^\top \Delta) = \|\Delta\|_F^2 .
    \label{eq:main_theorem}
\end{equation}
That is, a single random Gaussian probe gives an unbiased estimate of the squared Frobenius norm of the quantization error.
\end{theorem}

\begin{proof}
Let $A = \Delta^\top\Delta \in \mathbb{R}^{n \times n}$. Since $A$ is symmetric positive semidefinite,
\begin{align}
    \mathbb{E}[\mathbf{x}^\top A \mathbf{x}] &= \mathbb{E}\!\left[\sum_{i,j} A_{ij} x_i x_j\right] = \sum_{i,j} A_{ij}\, \mathbb{E}[x_i x_j] \\
    &= \sum_{i,j} A_{ij}\, \delta_{ij} = \sum_i A_{ii} = \mathrm{Tr}(A) = \mathrm{Tr}(\Delta^\top\Delta) = \|\Delta\|_F^2,
\end{align}
where the second equality uses $\mathbb{E}[x_i x_j] = \delta_{ij}$ for $\mathbf{x} \sim \mathcal{N}(\mathbf{0}, \mathbf{I})$.
\end{proof}

\begin{corollary}[Normalized probe sensitivity]
\label{cor:normalized}
The per-probe ratio $\|\Delta\mathbf{x}\|^2 / (\|\mathbf{W}\mathbf{x}\|^2 + \varepsilon)$ is a \emph{consistent} (but not unbiased, since $\mathbb{E}[A/B] \neq \mathbb{E}[A]/\mathbb{E}[B]$ in general) estimator of $\|\Delta\|_F^2 / \|\mathbf{W}\|_F^2 = \mathrm{NRMSE}^2$. The bias is $O(1/p)$ because $\|\mathbf{W}\mathbf{x}\|^2$ concentrates around $\|\mathbf{W}\|_F^2$.
\end{corollary}

\begin{remark}[Connection to Hutchinson's estimator]
Theorem~\ref{thm:main} is a special case of Hutchinson's trace estimator~\cite{hutchinson1989stochastic,bekas2007estimator}: $\mathrm{Tr}(A) = \mathbb{E}[\mathbf{z}^\top A \mathbf{z}]$ for any random vector $\mathbf{z}$ with $\mathbb{E}[\mathbf{z}\mathbf{z}^\top] = \mathbf{I}$. For Gaussian probes the variance is $\mathrm{Var}[\mathbf{x}^\top A\mathbf{x}] = 2\|A\|_F^2 = 2\|\Delta^\top\Delta\|_F^2$ (Lemma~\ref{lem:variance}), which is small relative to $\mathrm{Tr}(A)^2$ when $A$ has \emph{many} comparable eigenvalues, that is, when the error spectrum is flat. Section~\ref{sec:saturation} shows by direct measurement that quantization error is in this favorable regime.
\end{remark}

\begin{remark}[Why probe at all when $\|\Delta\|_F^2$ is computable?]
\label{rem:why_probes}
For an isolated tensor, $\|\Delta\|_F^2$ can be computed directly, and by Theorem~\ref{thm:main} the isotropic probe estimates nothing more than that. The probe formulation earns its place in two ways. First, it generalizes to non-identity input covariance: the allocator feeds each layer the probes propagated through the preceding layers (\S\ref{sec:signal}) and MLA blocks receive probes passed through their own compression path (\S\ref{sec:mla}), so the measured perturbation reflects the directions the network can produce rather than all $n$ directions equally. Second, the multi-probe sample variance is a free per-tensor uncertainty estimate, which flags the rare tensors whose error is not flat (\S\ref{app:deff_measurement}).
\end{remark}

\begin{proposition}[Weighted-input estimator]
\label{prop:weighted}
Let $\mathbf{x} \sim \mathcal{N}(\mathbf{0}, \Sigma)$ with $\Sigma \succeq 0$, and let the measured perturbation be $\boldsymbol{\delta} = \mathbf{M}\Delta\mathbf{x}$ for a fixed linear map $\mathbf{M}$. Write $\mathbf{K} = \mathbf{M}\Delta\Sigma^{1/2}$. Then
\begin{equation}
    \mathbb{E}\|\boldsymbol{\delta}\|^2 = \|\mathbf{K}\|_F^2 = \mathrm{Tr}\!\left(\Delta^\top \mathbf{M}^\top \mathbf{M} \Delta\, \Sigma\right), \qquad
    \mathrm{Var}\|\boldsymbol{\delta}\|^2 = 2\|\mathbf{K}^\top\mathbf{K}\|_F^2,
    \label{eq:weighted}
\end{equation}
so the single-probe coefficient of variation is $\sqrt{2/d_{\mathrm{eff}}(\mathbf{K}^\top\mathbf{K})}$.
\end{proposition}
\begin{proof}
Write $\mathbf{x} = \Sigma^{1/2}\mathbf{z}$ with $\mathbf{z} \sim \mathcal{N}(\mathbf{0}, \mathbf{I})$. Then $\|\boldsymbol{\delta}\|^2 = \mathbf{z}^\top \mathbf{K}^\top\mathbf{K}\, \mathbf{z}$, and Theorem~\ref{thm:main} and Lemma~\ref{lem:variance} apply with $A = \mathbf{K}^\top\mathbf{K}$.
\end{proof}

Two readings of Proposition~\ref{prop:weighted} matter for what follows. With $\mathbf{M} = \mathbf{I}$ and $\Sigma = \mathbb{E}[\mathbf{x}\mathbf{x}^\top]$ the second moment of a layer's inputs, $\|\mathbf{K}\|_F^2 = \mathrm{Tr}(\Delta \Sigma \Delta^\top)$ is the layer-wise reconstruction objective that GPTQ and OBQ minimize with calibration data, $\|\Delta \mathbf{X}\|_F^2 = \mathrm{Tr}(\Delta\,\mathbf{X}\mathbf{X}^\top\Delta^\top)$, up to the sample count. A probe drawn from the network's own propagated input distribution therefore estimates the calibration objective without data (\S\ref{sec:signal}). And with $\mathbf{M}$ the Jacobian of the rest of the layer at the reference input, the proposition describes the perturbation the allocator measures to first order in $\Delta$. Both readings come with a warning: the flatness measured in \S\ref{sec:saturation} is a property of $\Delta^\top\Delta$, and $d_{\mathrm{eff}}(\mathbf{K}^\top\mathbf{K})$ can be far smaller when $\Sigma$ or $\mathbf{M}$ concentrate energy in a few directions. Whether the propagated signal is calibrated is therefore a separate measurement (\S\ref{sec:propagated_calibration}).

\subsection{Spectral Flatness of Quantization Error}
\label{sec:saturation}

How many probes are needed? Theorem~\ref{thm:main} guarantees unbiasedness with a single probe, but if the error spectrum were concentrated on a few directions, a single probe would be noisy. The variance of the $p$-probe estimator is controlled by the \emph{effective dimensionality} of the error spectrum.

\begin{definition}[Effective dimensionality]
For a positive semidefinite matrix $A$ with eigenvalues $\lambda_1 \geq \lambda_2 \geq \cdots \geq \lambda_n \geq 0$, the effective dimensionality (spectral participation ratio) is
\begin{equation}
    d_{\mathrm{eff}}(A) = \frac{\mathrm{Tr}(A)^2}{\|A\|_F^2} = \frac{\left(\sum_i \lambda_i\right)^2}{\sum_i \lambda_i^2} .
\end{equation}
\end{definition}

By Lemma~\ref{lem:variance},
\begin{equation}
    \mathrm{Var}\!\left[\frac{1}{p}\sum_{j=1}^p \mathbf{x}_j^\top A \mathbf{x}_j\right] = \frac{2\|A\|_F^2}{p} = \frac{2\,\mathrm{Tr}(A)^2}{p \cdot d_{\mathrm{eff}}(A)},
    \label{eq:variance_deff}
\end{equation}
so the single-probe coefficient of variation is $\sqrt{2/d_{\mathrm{eff}}}$: the estimator concentrates when $d_{\mathrm{eff}}$ is \emph{large}, that is, when error energy is spread over many directions. A rank-one error would give $\mathrm{CV} = \sqrt{2} \approx 141\%$; a flat spectrum over $10^3$ directions gives $\mathrm{CV} \approx 4.5\%$.

\paragraph{Direct measurement.}
Rather than assuming a spectral shape, we measure it. For every weight tensor of Qwen3.5-35B-A3B (1{,}317 tensors: attention, linear attention, dense MLP, fused MoE experts, vision blocks) and Qwen3.5-9B (366 tensors), at each of $b \in \{2,3,4,8\}$ bits (RTN, group 64), we compute $\|\Delta\|_F^2$ and $\|\Delta^\top\Delta\|_F^2$ in closed form via the small-side Gram matrix, with no estimation involved. Three facts emerge (Table~\ref{tab:deff}, \S\ref{app:deff_measurement}):
\begin{enumerate}[leftmargin=*,topsep=2pt,itemsep=1pt]
    \item \textbf{Quantization error is spectrally flat.} For an $m \times n$ i.i.d.\ noise matrix, $d_{\mathrm{eff}}(\Delta^\top\Delta) = mn/(m{+}n)$ in expectation (the Marchenko--Pastur participation ratio). Measured quantization error attains $0.93\times$ this value at the median on the 35B MoE and $0.96\times$ on the 9B dense model (IQR $[0.86, 0.97]$ and $[0.90, 0.98]$). The largest tensors sit closest to the ceiling: the fused expert stacks $(262{,}144 \times 2{,}048)$ predict $d_{\mathrm{eff}} = 2{,}032.1$ and measure $2{,}031.7$.
    \item \textbf{Flatness is bit-width invariant.} On the 35B model the median $d_{\mathrm{eff}}$ at 2, 3, 4, and 8 bits is $386.1$, $386.0$, $386.1$, and $386.2$: the spectral \emph{shape} of RTN error does not change as its magnitude shrinks $64\times$. On the 9B model the four medians lie within $3.5\%$ of one another ($1{,}424$ to $1{,}476$).
    \item \textbf{Probe accuracy is predictable from shape alone.} Since $d_{\mathrm{eff}} \approx 0.93\,mn/(m{+}n)$, the per-probe CV $\sqrt{2/d_{\mathrm{eff}}}$ can be computed \emph{before} touching the weights. At 4-bit the measured per-probe CV matches the prediction with median ratio $0.999$ (IQR $[0.965, 1.033]$) on the 35B model and $1.000$ (IQR $[0.968, 1.032]$) on the 9B model, and the 200-probe mean matches the exact $\|\Delta\|_F^2$ with median ratio $0.9998$ and $1.0000$. The estimator is calibrated, not merely unbiased.
\end{enumerate}

Flatness is what RTN's structure implies: within each group of 64 weights the rounding residuals are near-independent, sub-quantization-step fluctuations, so $\Delta$ behaves like a noise matrix rather than a low-rank perturbation. The consequence is the title claim: a single random probe is not a noisy Monte-Carlo sample but a $4$ to $7\%$ \emph{measurement}, with error bars known in advance.

\subsection{Propagated Probes: The Allocation Signal}
\label{sec:signal}

The allocator scores tensors inside the full network rather than in isolation. It draws a batch of $p = 50$ probe sequences of $s = 8$ positions, $\mathbf{X} \in \mathbb{R}^{p \times s \times d}$ with entries i.i.d.\ $\mathcal{N}(0,1)$, feeds them to the first decoder layer, and thereafter feeds each layer the reference output of the layer before it. For each weight tensor $\mathbf{W}_i$ in layer $\ell$ and each candidate bit-width $b \in \{2,3,4,5,6,8\}$, the pipeline quantizes the tensor with RTN at group size 64, runs the layer on its input batch, compares the perturbed output $\hat{\mathbf{Y}}$ with the reference output $\mathbf{Y}$, and restores the original weight. The score is the cosine distance between the flattened outputs, averaged over probe sequences:
\begin{equation}
    \text{div}_i(b) = 1 - \frac{\langle \mathbf{Y}, \hat{\mathbf{Y}} \rangle}{\|\mathbf{Y}\|\,\|\hat{\mathbf{Y}}\|} .
    \label{eq:divergence}
\end{equation}
For a small perturbation $\boldsymbol{\delta} = \hat{\mathbf{Y}} - \mathbf{Y}$, the cosine distance equals $\|\boldsymbol{\delta}_\perp\|^2 / (2\|\mathbf{Y}\|^2)$ up to third-order terms, where $\boldsymbol{\delta}_\perp$ is the component of $\boldsymbol{\delta}$ orthogonal to $\mathbf{Y}$. It is therefore a normalized output-perturbation energy of the same kind as the NRMSE$^2$ of Corollary~\ref{cor:normalized}, with two differences that are the point of propagation: the input covariance is the one the preceding layers produce rather than the identity, and the perturbation is measured after the layer's nonlinearity and residual path rather than at the linear output. By Proposition~\ref{prop:weighted} the propagated probe is the same estimator applied to the input- and downstream-weighted error operator; \S\ref{sec:propagated_calibration} measures its per-probe variance and convergence, and \S\ref{sec:iso_vs_prop} shows that it ranks tensors differently from the isolated estimator.

Two secondary signals ride on the same pass. On the final decoder layer the pipeline also records the fraction of probe positions whose argmax token under the language-model head changes when the tensor is quantized (the \emph{flip rate}) and adds it to the score with weight $0.1$. Between layers it monitors the cosine distance between consecutive reference outputs; if three consecutive layers exceed three times the running median, it switches from propagated to independent probes for the remaining layers. This switch triggered on Qwen3.5-35B-A3B and Qwen3-30B-A3B and not on Qwen3-8B; on the hybrid Qwen3.8-27B it fires at layer 32, and \S\ref{sec:objective} shows that there it harms the estimate it was meant to protect.

\subsection{Divergence Guardrails}
\label{sec:safety}

Dense models expose a failure mode of budget-driven allocation: every tensor has small divergence, so the greedy solver fills a tight budget with 2-bit assignments whose individually small errors compound across dozens of tensors. RAM applies three guardrails to the candidate set before and after allocation:
\begin{enumerate}[leftmargin=*,topsep=2pt,itemsep=1pt]
    \item \textbf{Absolute 2-bit veto.} If $\text{div}_i(2) > 10^{-4}$, the 2-bit option is removed from tensor $i$'s candidate set. On the models we tested the threshold separates MoE expert stacks (2-bit divergence of order $2 \times 10^{-5}$) from dense MLP tensors (order $2 \times 10^{-4}$).
    \item \textbf{At most 30\% of parameters at 2-bit.} If the allocation exceeds this, the highest-divergence 2-bit tensors are upgraded until it does not.
    \item \textbf{At least 3.5 average bits.} If breached, every 2-bit tensor is upgraded to 4-bit.
\end{enumerate}
At budgets below about 30\% of the BF16 size on dense models and MoE models with at most 128 experts, we also raise the minimum candidate bit-width to 4 (\S\ref{sec:main_results}); on these models 3-bit expert or MLP tensors cost more perplexity than the bytes they save.

\subsection{Budget-Targeted Allocation via MCKP}
\label{sec:allocation}

A design goal of RAM is \emph{exact budget targeting}: the user specifies a memory limit (``fit in 24\,GB'') and the system returns the lowest-divergence allocation that satisfies it. Threshold-based approaches produce one output size regardless of the hardware, and uniform quantization offers a few discrete size points (4-bit, 8-bit).

Given per-tensor scores and guardrail-filtered candidate sets $\mathcal{C}_i$, RAM solves a \textbf{Multiple-Choice Knapsack Problem (MCKP)}:
\begin{equation}
    \min_{\{b_i\}} \sum_{i=1}^{T} \text{div}_i(b_i) \qquad
    \text{s.t.} \quad \sum_{i=1}^{T} \text{size}_i(b_i) \leq (1-\rho)\,B, \quad b_i \in \mathcal{C}_i,
    \label{eq:mckp}
\end{equation}
where $\text{size}_i(b) = n_i b / 8 + \lfloor n_i / g \rfloor \cdot 4$ bytes for $b < 16$ (packed weights plus one 16-bit scale and bias per group of $g = 64$), $B$ is the user's byte budget, and $\rho = 0.1$ reserves a share of the budget for the tensors the allocator does not score: embeddings, the language-model head, and normalization weights, which the converter leaves at 16-bit, plus runtime metadata. Tensors with fewer than 1{,}024 elements are also left at 16-bit.

\paragraph{Greedy solver.}
Every tensor starts at its lowest admissible bit-width. The solver then repeatedly applies the upgrade $(b_i \to b_i')$ with the largest divergence reduction per added byte, $\eta = \Delta\text{div} / \Delta\text{size}$, until no upgrade fits the remaining budget. Because the scores are budget-independent, the same probe manifest is re-solved for any target size; the Qwen3.5-35B-A3B builds at 21, 23, 25, and 34\,GB in \S\ref{sec:downstream} all reuse one 188-second probe pass.

\subsection{Pipeline Summary}

\begin{algorithm}[t]
\caption{RAM: propagated-probe analysis and budget-targeted allocation}
\label{alg:pipeline}
\begin{algorithmic}[1]
\REQUIRE Model directory, byte budget $B$, probes $p$, positions $s$, group size $g$
\ENSURE Per-tensor manifest $\{(i, b_i)\}$
\STATE \textbf{Phase 1: Probe sensitivity analysis}
\STATE $\mathbf{X} \leftarrow$ $p \times s$ probe vectors $\sim \mathcal{N}(\mathbf{0}, \mathbf{I}_d)$
\FOR{each decoder layer $\ell$ (loaded one at a time in lazy mode)}
    \STATE $\mathbf{Y} \leftarrow \text{layer}_\ell(\mathbf{X})$ \COMMENT{reference output}
    \FOR{each weight tensor $\mathbf{W}_i$ in layer $\ell$ with $\geq 1{,}024$ elements}
        \FOR{each candidate bit-width $b \in \{2, 3, 4, 5, 6, 8\}$}
            \STATE $\hat{\mathbf{W}}_i \leftarrow Q(\mathbf{W}_i; b, g)$; \; $\hat{\mathbf{Y}} \leftarrow \text{layer}_\ell(\mathbf{X})$ with $\hat{\mathbf{W}}_i$ in place; restore $\mathbf{W}_i$
            \STATE $\text{div}_i(b) \leftarrow 1 - \langle \mathbf{Y}, \hat{\mathbf{Y}}\rangle / (\|\mathbf{Y}\|\|\hat{\mathbf{Y}}\|)$ \quad (+ $0.1 \times$ flip rate on the last layer)
        \ENDFOR
    \ENDFOR
    \STATE $\mathbf{X} \leftarrow \mathbf{Y}$ unless the adaptation monitor has switched to independent probes
    \STATE Free layer weights
\ENDFOR
\STATE \textbf{Phase 2: Guardrails} \quad $\mathcal{C}_i \leftarrow \{b : b \neq 2 \text{ or } \text{div}_i(2) \leq 10^{-4}\}$
\STATE \textbf{Phase 3: Allocation} \quad Solve Eq.~\ref{eq:mckp} greedily; enforce the 30\% 2-bit cap and 3.5-bit floor
\STATE Write manifest JSON
\end{algorithmic}
\end{algorithm}

The pipeline (Algorithm~\ref{alg:pipeline}) processes decoder layers sequentially. In lazy mode it loads one layer's weights at a time and frees them after scoring, so models larger than memory can be analysed (Table~\ref{tab:timing}). The output manifest maps each tensor to its assigned bit-width and is consumed as a quantization predicate during model conversion.

\subsection{MLA-Aware Probe Coverage}
\label{sec:mla}

Multi-head Latent Attention (MLA)~\cite{deepseekai2024deepseekv2}, used in GLM-4.7-Flash and DeepSeek-V2, compresses key-value projections through a low-rank bottleneck: $\mathbf{K} = \mathbf{W}_K \mathbf{W}_{\text{compress}} \mathbf{x}$, where $\mathbf{W}_{\text{compress}} \in \mathbb{R}^{d_c \times d}$ projects to a compressed latent space ($d_c \ll d$). A probe drawn in the $d_c$-dimensional latent space with isotropic covariance spends its energy on all $d_c$ directions with equal weight, whereas at inference the up-projection $\mathbf{W}_K$ only ever sees the normalized image of $\mathbf{W}_{\text{compress}}$. The mismatch mis-estimates the sensitivity of the post-bottleneck weights, and on GLM-4.7-Flash the resulting allocation is markedly worse (\S\ref{sec:mla_results}).

RAM addresses this with \emph{two-stage probing}. For each MLA attention block: (1)~the compression weights $\mathbf{W}_{\text{compress}}$ are probed with the standard $d$-dimensional propagated probes (they see the full hidden dimension). (2)~For the up-projection weights $\mathbf{W}_K \in \mathbb{R}^{d_{\text{out}} \times d_c}$ that operate in the compressed space, we compute \emph{latent probes} by passing the propagated probes through the block's own compression weight and latent normalization layer:
\begin{equation}
    \mathbf{x}_{\text{latent}} = \mathrm{Norm}_{\text{latent}}\!\left(\mathbf{W}_{\text{compress}}\, \mathbf{x}\right).
\end{equation}
These $d_c$-dimensional latent probes score $\mathbf{W}_K$ via $\|\mathbf{W}_K \mathbf{x}_{\text{latent}} - \hat{\mathbf{W}}_K \mathbf{x}_{\text{latent}}\|^2$. Weights whose input dimension equals the hidden size (for example the output projection) keep full-dimensional probes. After probing all tensors, the full unquantized block output propagates to the next layer. MLA detection is automatic: if the model config contains \texttt{kv\_lora\_rank} $< d$, two-stage probing is activated. On GLM-4.7-Flash, 94 of 701 scored tensors received latent probes.

% ==============================================================================
\section{Experiments}
\label{sec:experiments}
% ==============================================================================

We evaluate RAM on seven model families using WikiText-2 perplexity (test split, 128 sequences of 2{,}048 tokens, seed 42). All experiments run on one Apple Mac Studio with an M2 Ultra (192\,GB unified memory); probing, conversion, and inference use MLX~\cite{apple2023mlx}. We report \textbf{median per-sequence perplexity} as the primary metric because a few outlier sequences dominate the mean (\S\ref{sec:median_ppl}). Unless stated otherwise, RAM builds use 50 propagated probe sequences of 8 positions, RTN at group size 64, and the six candidate bit-widths of \S\ref{sec:signal}.

\textbf{Models.} Three categories: (i) \emph{MoE}: Qwen3.5-35B-A3B~\cite{qwen2025qwen35} (256 experts), Qwen3-30B-A3B (128 experts), MiniMax-M2.5 (256 experts, native FP8), Llama-4-Scout~\cite{llama4} (16 experts), Qwen3.5-122B-A10B (128 experts); (ii) \emph{MoE with MLA}: GLM-4.7-Flash; and (iii) \emph{dense}: Qwen3-8B. Quality is evaluated on these seven models (8B to 122B parameters); probe timing is additionally reported for Llama-4-Maverick (400B parameters, 803\,GB).

\textbf{Baselines.} Uniform 4-bit RTN builds with group size 64 or 128, uniform 8-bit, a GPTQ-style ``protect attention'' allocation, and a calibration-based AWQ build on the dense model.

\subsection{Main Results}
\label{sec:main_results}

Table~\ref{tab:main} presents the comparison across all seven architectures. Uniform 4-bit has one fixed size per model, so the two builds are not always size-matched; the table lists both sizes. On the four MoE models whose uniform build is within 13\% of the RAM size (Qwen3.5-35B-A3B, MiniMax-M2.5, Llama-4-Scout, GLM-4.7-Flash), RAM lowers median perplexity by 3.5 to 13.6\%. On Qwen3.5-122B-A10B and Qwen3-8B the RAM builds are larger than uniform 4-bit and the comparison measures the practitioner's chosen budget rather than the allocator. On Qwen3-30B-A3B the tight budget with the 4-bit minimum admits no upgrades, so the ``RAM'' build is uniform 4-bit at group size 64 with unquantized embeddings; its advantage over the group-128 baseline is a group-size effect, not a mixed-precision one.

\begin{table}[t]
\centering
\caption{Main results across seven architectures. WikiText-2 test, 128 sequences, 2{,}048 tokens, seed 42; median perplexity. $\Delta$ is relative to the uniform 4-bit build in the same row; the uniform size is listed because uniform 4-bit builds have a fixed size. Probe time is wall-clock seconds for the full analysis pass on the M2 Ultra. $^{a}$Built with the 4-bit minimum (\S\ref{sec:safety}); the allocation is uniform 4-bit at group size 64. $^{b}$Uniform build evaluated on 256 sequences. $^{c}$RAM build is not size-matched to the uniform build. Uniform group size: g128 for Qwen3.5-35B-A3B and Qwen3.5-122B-A10B, g64 for Qwen3-8B, community 4-bit builds elsewhere.}
\label{tab:main}
\small
\resizebox{\textwidth}{!}{%
\begin{tabular}{llrrrrrrl}
\toprule
Model & Arch & BF16 (GB) & RAM (GB) & RAM Med.\ PPL & Unif.\ 4b (GB) & Unif.\ 4b Med.\ PPL & $\Delta$ & Time \\
\midrule
Qwen3.5-35B-A3B & MoE 256exp & 69.3 & 19.5 & 6.585 & 17.2 & 6.928 & $-$4.9\% & 188\,s \\
Qwen3-30B-A3B$^{a}$ & MoE 128exp & 61.1 & 16.8 & 8.877 & --- & 9.158 & $-$3.1\% & 53\,s \\
MiniMax-M2.5 & MoE 256exp FP8 & 230.1 & 112.3 & 9.070 & 119.8 & 9.399$^{b}$ & $-$3.5\% & 301\,s \\
Llama-4-Scout & MoE 16exp & 217.3 & 56.0 & 7.806 & 56.9 & 8.219$^{b}$ & $-$5.0\% & 214\,s \\
Qwen3.5-122B-A10B$^{c}$ & MoE 128exp & 250.2 & 107.6 & 5.327 & 60.4 & 5.601 & $-$4.9\% & 345\,s \\
GLM-4.7-Flash & MoE + MLA & 60.0 & 16.2 & 8.700 & --- & 10.075 & $-$13.6\% & 99\,s \\
Qwen3-8B$^{c}$ & Dense & 15.3 & 7.3 & 9.628 & 4.3 & 9.982 & $-$3.5\% & 36\,s \\
\bottomrule
\end{tabular}}
\end{table}

\paragraph{Scaling with model size.} The method runs unchanged from a 15\,GB dense model to 250\,GB MoE models (Qwen3.5-122B-A10B; MiniMax-M2.5 at 230\,GB in FP8). The largest gain, 13.6\% on GLM-4.7-Flash, comes from MLA-aware probe coverage (\S\ref{sec:mla_results}).

\paragraph{Budget targeting.} Because the scores are budget-independent, the allocator is re-run at any target size from one manifest. The Qwen3.5-35B-A3B builds at four budgets in \S\ref{sec:downstream} share a single 188-second probe pass. Uniform quantization offers a few discrete size points; GPTQ and AWQ require a calibration run per target size.

\paragraph{Probe time.} The analysis pass takes 36 seconds on the smallest model (Qwen3-8B) and 345 seconds on the largest evaluated for quality (Qwen3.5-122B-A10B, 250\,GB). An 803\,GB model (Llama-4-Maverick) is analysed in 539 seconds (Table~\ref{tab:timing}).

\subsection{Comparison with a GPTQ-Style Heuristic}
\label{sec:gptq_comparison}

A common strategy in GPTQ and AWQ deployments is to protect attention (keep it at high precision) while quantizing the MLP and expert tensors. We test this rule against RAM's allocation on Qwen3.5-35B-A3B. The GPTQ-style build leaves attention and shared-expert tensors unquantized and applies 4-bit RTN at group size 128 everywhere else.

\begin{table}[t]
\centering
\caption{RAM vs.\ a GPTQ-style ``protect attention'' allocation on Qwen3.5-35B-A3B. All builds use RTN; the comparison isolates the \emph{allocation}. The GPTQ-style build keeps attention and shared-expert tensors at 16-bit and everything else at 4-bit g128.}
\label{tab:gptq}
\begin{tabular}{lrrr}
\toprule
Method & Size (GB) & Median PPL & $\Delta$ vs BF16 \\
\midrule
BF16 & 69.3 & 6.494 & --- \\
RAM (probe allocation) & 19.5 & 6.585 & $+$1.4\% \\
GPTQ-style (protect attention) & 20.5 & 6.697 & $+$3.1\% \\
Uniform 4-bit g64 & 17.4 & 6.764 & $+$4.2\% \\
Uniform 4-bit g128 & 17.2 & 6.928 & $+$6.7\% \\
\bottomrule
\end{tabular}
\end{table}

RAM reaches $+$1.4\% degradation vs BF16 at 19.5\,GB; the GPTQ-style build reaches $+$3.1\% at the \emph{larger} size of 20.5\,GB. The heuristic spends its bytes on a tensor class: it lifts every attention and shared-expert tensor to 16-bit, including those the probes rank as tolerant, and leaves the routed experts at a flat 4-bit. RAM spends the same budget on specific tensors. Its Qwen3.5-35B-A3B manifest keeps 93 of 260 attention tensors and all 40 routers at 16-bit, puts most shared-expert tensors at 8-bit, and spreads the routed expert stacks over 2, 3, 4, and 5 bits (Appendix~\ref{app:allocation}). Section~\ref{sec:sensitivity_patterns} reports the sensitivity ordering behind this allocation.

\subsection{Probe Convergence Against Exact Ground Truth}
\label{sec:saturation_experiment}

Because $\|\Delta\|_F^2$ is computable in closed form, convergence of the isolated-tensor estimator can be measured against \emph{exact} ground truth rather than against a converged probe reference. For each tensor we draw 200 independent probes, form $p$-probe estimates of $\mathrm{NRMSE}^2 = \|\Delta\|_F^2/\|\mathbf{W}\|_F^2$ from disjoint probe blocks, and compare the induced cross-tensor ranking with the exact one (Table~\ref{tab:saturation}).

\begin{table}[t]
\centering
\caption{Probe convergence against \emph{exact} closed-form ground truth at 4-bit (disjoint probe blocks; mean over blocks). Relative MAE of the sensitivity estimate decays with log-log slope $-0.50$ on Qwen3.5-35B-A3B and $-0.51$ on Qwen3.5-9B, the i.i.d.\ averaging rate, from a single-probe base set by $\sqrt{2/d_{\mathrm{eff}}}$. Rank agreement is limited by the tightness of the sensitivity distribution rather than by probe noise: at 4-bit the normalized sensitivities of all tensors span $0.41$ dex (sd $0.065$ dex) on the 35B model and $0.36$ dex (sd $0.049$ dex) on the 9B model.}
\label{tab:saturation}
\small
\begin{tabular}{rcccc}
\toprule
& \multicolumn{2}{c}{Qwen3.5-35B-A3B (1{,}317 tensors)} & \multicolumn{2}{c}{Qwen3.5-9B (366 tensors)} \\
\cmidrule(lr){2-3}\cmidrule(lr){4-5}
Probes ($p$) & Spearman $\rho$ & Rel.\ MAE & Spearman $\rho$ & Rel.\ MAE \\
\midrule
1 & 0.671 & 6.2\% & 0.627 & 5.5\% \\
5 & 0.855 & 2.8\% & 0.864 & 2.5\% \\
10 & 0.903 & 2.0\% & 0.921 & 1.8\% \\
20 & 0.938 & 1.4\% & 0.957 & 1.3\% \\
50 & 0.968 & 0.90\% & 0.982 & 0.75\% \\
100 & 0.981 & 0.64\% & 0.991 & 0.53\% \\
\bottomrule
\end{tabular}
\end{table}

\paragraph{Convergence is i.i.d.}
The relative MAE decays with log-log slope $-0.50$ (35B) and $-0.51$ (9B), the textbook $p^{-1/2}$ rate. An earlier version of this work reported a slope of $-1.43$; that number was an artifact of measuring convergence against a 500-probe reference whose probe bank \emph{contained} the evaluated probes, so estimate and reference shared samples and the MAE fell to zero at $p = 500$ by construction. Measured against exact ground truth with disjoint probes there is no super-convergence, and none is needed: the single-probe estimate already starts at $6\%$ relative error because $d_{\mathrm{eff}}$ is large, and 20 probes reach $1.4\%$.

\paragraph{Rank agreement is spread-limited, not noise-limited.}
At 4-bit the exact normalized sensitivities are tightly clustered (sd $0.065$ dex, range $0.41$ dex on the 35B model). With 20 probes the per-tensor estimate error ($1.4\%$) is far below this spread, so the surviving rank disagreements ($\rho = 0.94$) are swaps between tensors whose true sensitivities differ by less than the estimate error. Whether such swaps matter for the final model is an end-to-end question that this measurement does not answer; \S\ref{sec:limitations} returns to it.

\paragraph{Practical implication.}
For the isolated estimator, practitioners can bound the probe count needed for a target accuracy in advance from tensor shapes alone, via $\sqrt{2/(p\,d_{\mathrm{eff}})}$ with $d_{\mathrm{eff}} \approx 0.93\,mn/(m{+}n)$.

\subsection{Calibration of the Propagated Signal}
\label{sec:propagated_calibration}

Proposition~\ref{prop:weighted} says the propagated probe is a Hutchinson estimator of the weighted operator $\mathbf{K}$; whether it is calibrated depends on $d_{\mathrm{eff}}(\mathbf{K}^\top\mathbf{K})$, which involves the network's own input second moment and downstream Jacobian and is not available in closed form. It can be measured. We re-ran the propagated pass of \S\ref{sec:signal} with 200 probe sequences on three models, recording the per-sequence divergence of every tensor at 2, 4, and 8 bits, and scored $p$-sequence estimates from the first 100 sequences against the mean of the other 100 (Table~\ref{tab:propagated_calibration}; 2- and 8-bit agree with 4-bit to the digits shown).

\begin{table}[t]
\centering
\caption{Calibration of the propagated signal at 4-bit (200 probe sequences of 8 positions) next to the isolated estimator from \S\ref{sec:saturation}. CV is the per-probe-sequence coefficient of variation of a tensor's divergence (median over tensors, IQR in brackets); $d_{\mathrm{eff}}$ is $2/\mathrm{CV}^2$ for the propagated rows and the measured value for the isolated rows. Spread is the standard deviation of $\log_{10}$ of the signal across tensors. $\rho_p$ is the Spearman rank correlation between a $p$-probe estimate and a disjoint reference; MAE$_{50}$ is the median relative error at $p = 50$, the allocator's setting. Qwen3-8B was not part of the isolated measurement.}
\label{tab:propagated_calibration}
\small
\resizebox{\textwidth}{!}{%
\begin{tabular}{llrrrrrrrr}
\toprule
Model & Signal & Tensors & CV [IQR] & $d_{\mathrm{eff}}$ & $\mathrm{CV}(\|\mathbf{Y}_j\|)$ & Spread (dex) & $\rho_1$ & $\rho_{50}$ & MAE$_{50}$ \\
\midrule
Qwen3.5-9B & isolated & 366 & 0.037 & 1{,}476 & --- & 0.049 & 0.627 & 0.982 & 0.75\% \\
Qwen3.5-9B & propagated & 312 & 0.16 [0.09, 0.22] & 79 & 0.04 & 1.06 & 0.994 & 1.000 & 1.3\% \\
Qwen3.5-35B-A3B & isolated & 1{,}317 & 0.072 & 386 & --- & 0.065 & 0.671 & 0.968 & 0.90\% \\
Qwen3.5-35B-A3B & propagated & 590 & 0.24 [0.13, 0.59] & 34 & 0.03 & 0.86 & 0.954 & 0.998 & 2.6\% \\
Qwen3-8B & propagated & 324 & 2.45 [1.60, 3.16] & 0.3 & 1.67 & 0.64 & 0.742 & 0.967 & 19\% \\
\bottomrule
\end{tabular}}
\end{table}

\paragraph{The cosine signal is the perturbation energy of Proposition~\ref{prop:weighted}.} The first-order identity $1 - \cos = \|\boldsymbol{\delta}_\perp\|^2 / (2\|\mathbf{Y}\|^2)$ holds per sequence to a median relative deviation of $2 \times 10^{-4}$ or better on all three models, and the two quantities rank tensors identically (Spearman $\geq 0.995$).

\paragraph{Calibration transfers only partly, and the shortfall is model-dependent.} On Qwen3.5-9B and Qwen3.5-35B-A3B the propagated operator is far less flat than the isolated error ($d_{\mathrm{eff}}$ 79 against 1{,}476 and 34 against 386), so a single sequence measures a tensor to 16 to 24\% rather than 4 to 7\%; fifty sequences still bring the median tensor within 1.0 to 1.7\% of the 200-sequence value. On Qwen3-8B the propagated signal is not calibrated at all: the per-sequence CV is 2.4, an implied $d_{\mathrm{eff}}$ below one, and the 50-sequence mean carries a median error of 19\% (90th percentile 41\%). The cause is visible in the reference outputs. From the third decoder layer onward the residual-stream norm of Qwen3-8B varies by a factor of 1.7 across Gaussian probe sequences (Table~\ref{tab:propagated_calibration}, $\mathrm{CV}(\|\mathbf{Y}_j\|)$), against 0.03 to 0.04 on the two Qwen3.5 models: a few sequences excite very large activations, and the per-sequence divergence inherits their heavy tail. This is the concentration of energy in few directions that Proposition~\ref{prop:weighted} warns about, produced by $\Sigma$ and $\mathbf{M}$ rather than by $\Delta$.

\paragraph{Ranking survives because the spread is wide.} The propagated signal spans 0.64 to 1.06 dex across tensors, ten to twenty times the 0.05 to 0.07 dex of the isolated estimator, so per-tensor noise that would scramble the isolated ranking leaves the propagated ranking intact. At the allocator's 50 sequences the rank agreement with a disjoint reference is 0.998 (35B), 1.000 (9B), and 0.967 (8B); a single sequence already gives 0.95 and 0.99 on the two Qwen3.5 models. The relative error decays with probe count at the i.i.d.\ rate on all three (log-log slopes $-0.50$, $-0.45$, and $-0.46$ for the 9B, 35B, and 8B models), as Proposition~\ref{prop:weighted} predicts for any fixed operator. The isolated estimator is calibrated but narrow; the propagated signal is noisy but wide. For models that develop large activations under random inputs, more sequences or a median aggregate would be the remedy; we have not tested either.

\subsection{The Propagated Probe Estimates the Calibration Objective}
\label{sec:objective}

Proposition~\ref{prop:weighted} makes a claim that can be checked directly: if the propagated probes carry the network's own input statistics, the quadratic form they estimate should track the layer-wise objective that calibration-based methods compute from real activations. We measured both on every linear tensor of Qwen3-8B, Qwen3.5-9B, and the 27B hybrid Qwen3.8-27B. The real objective is $Q_{\text{real}}(t,b) = \sum_x \|\Delta_b \mathbf{x}\|^2 / \sum_x \|\mathbf{W}\mathbf{x}\|^2$ over the inputs the tensor receives when the model runs causally on 32 sequences of 512 tokens from the WikiText-2 training split (never the test split used for perplexity), with a second real-text domain (Tulu-3 SFT dialogue) as a ceiling on how far the objective itself moves between domains. $Q_{\text{probe}}$ is the same ratio over the pipeline's propagated probes (200 sequences of 8 positions, no mask, same adaptation rule), and $Q_{\text{iso}} = \|\Delta_b\|_F^2/\|\mathbf{W}\|_F^2$ is the isolated estimator. Table~\ref{tab:objective} reports rank correlations across tensors at 4-bit.

\begin{table}[t]
\centering
\caption{Does the propagated probe estimate the calibration objective? Spearman rank correlation across linear tensors at 4-bit between the objective under real activations ($Q_{\text{real}}$, WikiText-2 train) and each signal. ``Domain ceiling'' is $Q_{\text{real}}$ on WikiText-2 against $Q_{\text{real}}$ on Tulu-3 dialogue. ``Block cosine'' is the pipeline's own signal (\S\ref{sec:signal}), measured at the layer output. Ratio is the median per-tensor $Q_{\text{probe}}/Q_{\text{real}}$ with IQR.}
\label{tab:objective}
\small
\resizebox{\textwidth}{!}{%
\begin{tabular}{lrrrrrrr}
\toprule
Model & Tensors & $Q_{\text{probe}}$ & $Q_{\text{iso}}$ & HAWQ-V2 & Block cosine & Domain ceiling & Ratio \\
\midrule
Qwen3-8B & 252 & 0.83 & $-0.07$ & 0.25 & 0.50 & 0.99 & 1.29 [1.05, 1.48] \\
\quad attention & 144 & 0.90 & $-0.03$ & 0.36 & --- & --- & 1.35 \\
\quad MLP & 108 & 0.71 & 0.08 & 0.07 & --- & --- & 1.13 \\
Qwen3.5-9B & 248 & 0.81 & $-0.06$ & --- & 0.11 & 0.99 & 1.12 [0.92, 1.43] \\
\quad attention & 32 & 0.88 & $-0.04$ & --- & --- & --- & 0.99 \\
\quad linear attention & 120 & 0.72 & $-0.05$ & --- & --- & --- & 1.25 \\
\quad MLP & 96 & 0.87 & 0.08 & --- & --- & --- & 1.05 \\
\midrule
Qwen3.8-27B, adaptation rule on (fired at layer 32) & 496 & 0.27 & $-0.15$ & --- & --- & --- & 1.19 [0.96, 1.76] \\
\quad layers 0--31 (propagated) & 248 & 0.79 & $-0.11$ & --- & --- & --- & 1.14 \\
\quad layers 32--63 (independent probes) & 248 & $-0.22$ & $-0.17$ & --- & --- & --- & 1.62 \\
Qwen3.8-27B, adaptation rule off & 496 & 0.83 & $-0.15$ & --- & --- & --- & 1.15 [0.89, 1.34] \\
\quad attention & 64 & 0.91 & 0.23 & --- & --- & --- & 1.48 \\
\quad linear attention & 240 & 0.76 & $-0.23$ & --- & --- & --- & 1.06 \\
\quad MLP & 192 & 0.80 & 0.10 & --- & --- & --- & 1.16 \\
\quad layers 0--31 & 248 & 0.79 & $-0.11$ & --- & --- & --- & 1.14 \\
\quad layers 32--63 & 248 & 0.85 & $-0.17$ & --- & --- & --- & 1.15 \\
\bottomrule
\end{tabular}}
\end{table}

\paragraph{Propagation supplies the second moment.} With real activations the ranking of tensors by quantization damage is stable across domains (0.99), and the propagated probe reproduces it at $0.83$ and $0.81$ on the two smaller models and at $0.83$ on the 27B hybrid once its adaptation rule is switched off (0.88 to 0.91 within full attention), with no data at all. The isolated estimator, which is what Theorem~\ref{thm:main} certifies, is uncorrelated with the real objective ($-0.07$, $-0.06$, $-0.15$): the second moment of the inputs, not the rounding error, decides which tensors matter, and only the propagated probe has it. HAWQ-V2's Hessian trace, which folds in the loss curvature, tracks the layer objective at 0.25 on Qwen3-8B, so the matched-perplexity tie of \S\ref{sec:hawqv2} is between two signals that both see the input statistics and differ in what else they weight. The probe over-estimates the objective in absolute terms by 12 to 29\% at the median, because Gaussian probes over-excite the few high-energy channels of the residual stream (the top eight channels carry 29\% of probe energy against 17\% of real energy on Qwen3-8B, and 28\% against 22\% on Qwen3.8-27B), but the ranking survives.

\paragraph{The adaptation rule destroys the estimate, and switching it off restores it.} On Qwen3.8-27B the pipeline's adaptation monitor (\S\ref{sec:signal}) fired at layer 32 and replaced the propagated probes with fresh isotropic ones for the remaining 32 layers. With the rule on, the overall correlation is only $0.27$, and the reason is visible by depth: $0.79$ for the propagated half, $-0.22$ for the independent half, where the probe signal collapses onto the isolated estimator ($\rho = 0.75$ between them), which is what an isotropic input must produce. We re-ran the pass with the rule disabled and nothing else changed. The overall correlation rises to $0.83$, the value measured on Qwen3-8B; layers 32--63 move from $-0.22$ to $0.85$ and layers 0--31 stay at $0.79$, where the rule had not yet fired. Within classes the estimate is $0.91$ for full attention, $0.80$ for the MLP and $0.76$ for the 240 gated-DeltaNet projections; the one projection it does not track is the $\beta$-gate input ($0.06$ over 48 tensors of shape $48 \times 5120$, 0.05\% of the linear parameters). The switch was introduced to stabilise the block-output cosine on hybrid stacks; for the objective estimate it is exactly wrong, because it throws away the input statistics that make the estimate work. The under-excitation of high-energy channels we had first attributed to this model was an artefact of the same rule: with propagated inputs throughout, the top eight channels carry 28\% of probe energy against 22\% of real energy, the same mild over-excitation as on Qwen3-8B.

\paragraph{A better estimate of the objective is not a better allocation signal.} The pipeline's block-output cosine tracks the real objective at only $0.50$ (8B) and $0.11$ (9B), against $0.83$ and $0.81$ for the propagated quadratic form at the linear output, and the quadratic form is cheaper: for a stored input batch it is one matrix product per candidate bit-width, with no re-execution of the layer. The natural next step is to allocate with it, and we tested that on Qwen3.8-27B in a setting that reuses none of the pipeline's own build path. Each signal goes to the same greedy knapsack over llama.cpp's k-quant types (Q2\_K to Q8\_0, one type per tensor); the budget is bisected until the file size matches llama.cpp's own IQ3\_M mix within 1.7\%; \texttt{llama-quantize} then builds the model from the vendor BF16 file with one importance matrix (100 chunks of 512 tokens) shared by every arm. The signals decide the 304 tensors of the feed-forward blocks, the sixteen full-attention blocks and the DeltaNet output projections (10.4 of the 12.6 GB); embeddings, output head, norms and, through a tensor-naming gap in our bridge that we found after the runs, the 192 DeltaNet input projections (1.9 GB) take llama.cpp's default types, identical in every arm. With the naming fixed and the file matched to the IQ3\_M byte count exactly, the ensemble arm reaches 6.018 whether the signal controls those projections or llama.cpp does (6.020), from a manifest probed without the adaptation rule; with the rule on, the same rebuild reaches 6.115, because the isotropic probes it substitutes from layer 32 under-protect the mid-depth feed-forward and DeltaNet gate tensors. Perplexity is on the first 40 chunks of 2{,}048 tokens of WikiText-2 test. Every arm is scored on the same chunks, so Table~\ref{tab:gguf_signals} reports, next to the mean, the number of chunks on which an arm beats the vendor mix and the number on which it beats the best arm. The seven arms are the two vendor mixes; the pipeline's propagated cosine; that cosine combined with an isolated-tensor structural signal, the geometric mean of the cosine and one minus the linear centered kernel alignment~\cite{kornblith2019cka} between a tensor's outputs before and after quantization on 32 Gaussian probes, which is the signal behind the released builds of this model; the propagated quadratic form with the adaptation rule on and off; and $Q_{\text{real}}$ itself, an oracle that allocates from the real-activation objective computed on WikiText-2 training text.

\begin{table}[t]
\centering
\caption{Allocating from each signal at the byte count of llama.cpp's IQ3\_M mix, Qwen3.8-27B, k-quant types with a shared importance matrix. Perplexity on 40 chunks of 2{,}048 tokens of WikiText-2 test (mean $\pm$ standard error over chunks). ``Chunks'' counts the chunks, out of 40, on which the arm beats IQ3\_M and on which it beats the cosine-and-CKA arm. The oracle is the only arm that uses text.}
\label{tab:gguf_signals}
\small
\resizebox{\textwidth}{!}{%
\begin{tabular}{llrrrr}
\toprule
Signal & Data & File (GB) & Perplexity & vs.\ IQ3\_M & vs.\ best \\
\midrule
IQ3\_M (llama.cpp mix) & none & 12.58 & 6.336 $\pm$ 0.073 & --- & 3 \\
Q3\_K\_M (llama.cpp mix) & none & 13.30 & 6.207 $\pm$ 0.075 & 31 & 3 \\
Propagated block cosine & none & 12.72 & 6.041 $\pm$ 0.071 & 36 & 6 \\
Block cosine $\times$ isolated CKA & none & 12.79 & \textbf{5.991} $\pm$ 0.071 & 37 & --- \\
Quadratic form, adaptation rule on & none & 12.61 & 6.204 $\pm$ 0.075 & 34 & 0 \\
Quadratic form, adaptation rule off & none & 12.66 & 6.238 $\pm$ 0.076 & 30 & 1 \\
$Q_{\text{real}}$ (oracle) & WikiText-2 train & 12.80 & 6.087 $\pm$ 0.073 & 37 & 5 \\
\bottomrule
\end{tabular}}
\end{table}

The block-output signals win, and the oracle loses to them. The cosine-and-CKA arm reaches 5.991 against 6.336 for IQ3\_M at the same bytes, better on 37 of 40 chunks, and Q3\_K\_M needs 5.7\% more bytes to reach 6.207; the cosine alone reaches 6.041. The oracle allocates from the exact quantity that calibration-based methods optimize, reaches 6.087, and loses to the cosine-and-CKA arm on 35 of 40 chunks (Wilcoxon $p = 10^{-5}$) and to the cosine alone on 29 of 40 ($p = 0.004$). Its data-free estimate does worse still, at 6.238 with the adaptation rule off and 6.204 with it on, and the oracle beats the rule-off estimate on 39 of 40 chunks, so the estimate is imperfect as an allocation signal; but the gap that matters is the one between the oracle and the block-output signals, and it has the opposite sign from the one a rank correlation of 0.83 suggests.

The reason is what the two quantities normalize by. $Q_{\text{real}}$ divides a tensor's output perturbation by that tensor's own output energy, so it says how badly the tensor rounds and nothing about how much its output matters to the block: in the terms of Proposition~\ref{prop:weighted} it has $\mathbf{M} = \mathbf{I}$. The block cosine measures the perturbation against the block output, residual stream included, so a tensor whose contribution is small relative to the residual scores low however badly it rounds; it is Proposition~\ref{prop:weighted} with $\mathbf{M}$ the Jacobian of the rest of the block. The allocations show the difference. Over the tensors the signals control, the cosine-and-CKA arm spends 2.7 to 2.8 bits per weight on layers 0 to 31 and 4.3 to 4.5 on layers 32 to 63, and gives the sixteen full-attention blocks 4.1 bits; the oracle is nearly flat by depth (3.2 to 3.9 bits), gives full attention 3.4 bits, and spends on the DeltaNet output projections instead (4.4 against 3.8). The two agree on the type of 53\% of those tensors. So the propagated quadratic form is the right way to estimate the calibration objective without data, and the calibration objective is the wrong quantity to allocate from. The signal the allocator needs is the downstream-weighted form, which the block cosine approximates at the cost of a layer re-execution and which the quadratic form could compute at one matrix product if $\mathbf{M}$ were supplied. We have not built that, and the allocation test covers one model and one quantizer family.

\subsection{Dense Model at a Tight Budget: Guardrails and a Negative Result}
\label{sec:safety_experiment}

Table~\ref{tab:safety} reports Qwen3-8B (dense, 15.3\,GB) at a 4\,GB target and at a 6\,GB target. The 4\,GB target is a negative result for RAM. Both builds overshoot the target because the 10\% reserve of \S\ref{sec:allocation} under-estimates the cost of this model's unquantized embedding and language-model head (each $151{,}936 \times 4{,}096$ at 16-bit, 2.5\,GB together), and both are worse than uniform 4-bit g64 at a smaller 4.3\,GB. Within the pair, the build restricted to $\{2,4,8\}$ bits with the 2-bit veto beats the six-width build: on this dense model the 3-bit MLP tensors that the six-width solver prefers cost more perplexity than the bytes they save, which motivated the 4-bit minimum of \S\ref{sec:safety}. At the 6\,GB target (7.35\,GB actual, 48\% of BF16) the six-width allocation, 55\% of it at 5-bit, matches BF16 within $0.2\%$.

\begin{table}[t]
\centering
\caption{Qwen3-8B (dense) at a tight 4\,GB target and a 6\,GB target. Both 4\,GB builds overshoot the target and trail uniform 4-bit g64; the 6\,GB build matches BF16. ``2-bit'' is the share of parameters assigned 2-bit. Mean and median over the same 128 sequences.}
\label{tab:safety}
\begin{tabular}{lrrrr}
\toprule
Build & Size (GB) & 2-bit & Mean PPL & Median PPL \\
\midrule
BF16 & 15.3 & 0\% & 9.624 & 9.648 \\
Uniform 4-bit g64 & 4.3 & 0\% & 9.965 & 9.982 \\
RAM 4\,GB target, $\{2,4,8\}$ + veto & 5.8 & 7.4\% & 10.231 & 10.299 \\
RAM 4\,GB target, six widths + veto & 5.7 & 1.4\% & 10.473 & 10.494 \\
RAM 6\,GB target, six widths & 7.3 & 0\% & 9.642 & 9.628 \\
\bottomrule
\end{tabular}
\end{table}

\subsection{Allocating with the Isolated Estimator}
\label{sec:iso_ablation}

If the isolated estimator is calibrated and the propagated one is not, why allocate with the propagated one? We answer by allocation. On Qwen3.5-35B-A3B we computed the exact isolated NRMSE$^2$ of every scored tensor at all six candidate bit-widths with the pipeline's own quantizer, and fed it to the same greedy solver at the same 21\,GB budget with the same candidate sets (the 2-bit veto decisions were copied from the propagated manifest), so that the cost signal is the only difference. Both manifests were then converted and evaluated under one runtime (Table~\ref{tab:iso_ablation}).

\begin{table}[t]
\centering
\caption{Isolated versus propagated allocation on Qwen3.5-35B-A3B at a 21\,GB target: same solver, budget, quantizer, and candidate sets; only the cost signal differs. \emph{Upper}: parameter-weighted mean bit-width by tensor class. \emph{Lower}: the two builds converted and evaluated under the same runtime; the propagated row of Table~\ref{tab:main} is repeated for reference. The two allocations assign the same bit-width to 50\% of tensors.}
\label{tab:iso_ablation}
\small
\resizebox{0.98\textwidth}{!}{%
\begin{tabular}{lrrrrr}
\toprule
& Attention & Routers & Shared experts & Routed experts & Other \\
\midrule
Propagated signal (avg bits) & 8.5 & 16.0 & 8.0 & 3.8 & 11.8 \\
Isolated estimator (avg bits) & 5.7 & 8.0 & 8.0 & 3.9 & 10.6 \\
\midrule
\multicolumn{6}{l}{} \\
Build & \multicolumn{2}{r}{Size (GB)} & \multicolumn{2}{r}{Median PPL} & MMLU \\
\midrule
Propagated (Table~\ref{tab:main} build) & \multicolumn{2}{r}{19.5} & \multicolumn{2}{r}{6.585} & 0.657 \\
Propagated (rebuilt, same runtime) & \multicolumn{2}{r}{19.5} & \multicolumn{2}{r}{6.544} & 0.668 \\
Isolated estimator & \multicolumn{2}{r}{19.5} & \multicolumn{2}{r}{6.550} & 0.679 \\
\bottomrule
\end{tabular}}
\end{table}

The isolated estimator, ranking by rounding error per unit of weight energy, spends less on attention (5.7 against 8.5 average bits, with 30 attention tensors at 2-bit against 22) and on routers (8-bit against 16-bit), and returns the bytes to the routed expert stacks (3.9 against 3.8 bits; 109 of 120 stacks at 4-bit against 74). On WikiText-2 the two builds are indistinguishable: median perplexity $6.544$ against $6.550$ and mean $6.622$ against $6.625$ at the same $19.50$\,GB. That gap is a tenth of the shift the runtime upgrade alone produced on the propagated manifest ($6.585$ under mlx\_lm 0.30.4 in Table~\ref{tab:main} to $6.544$ under 0.31.3, same weights), so perplexity does not separate the two signals at this budget. Both agree on the coarse structure, small tensors up and expert stacks down, and the size-normalized upgrade rule of the greedy solver enforces that structure whichever signal supplies the costs. On MMLU the isolated allocation scores $0.679$ against $0.668$ for the propagated rebuild ($\pm 0.004$ each; the runtime upgrade moved the propagated build from $0.657$), a $1.1$-point margin of about two standard errors in favour of the calibrated isolated estimator, consistent with its expert stacks sitting at 4-bit rather than 3-bit (109 of 120 against 74). At this budget on this model, the signal that the theory covers allocates at least as well as the signal the pipeline used.

\subsection{Matched Comparison with HAWQ-V2}
\label{sec:hawqv2}

The natural calibration-based comparator is HAWQ-V2~\cite{dong2020hawqv2}, which scores tensor $t$ at bit-width $b$ by $\Omega_t(b) = (\mathrm{Tr}\,H_t / n_t)\,\|\mathbf{W}_t - Q_b(\mathbf{W}_t)\|_F^2$, with the Hessian trace estimated by Hutchinson's method through double back-propagation of the language-model loss on calibration text. Proposition~\ref{prop:weighted} says the propagated probe estimates a related but different quadratic form (the input second moment and downstream Jacobian, without the loss curvature), so the comparison asks whether the loss curvature buys anything at matched bytes.

\paragraph{Protocol.} We use our HAWQ-V2 implementation from earlier work: 16 WikiText-2 sequences of 256 tokens, 8 Rademacher probes each (128 Hessian--vector products), bfloat16, one seed, and the same group-64 RTN quantizer for the reconstruction term. The comparison runs on Qwen3-8B, the paper's dense model, at its 6\,GB target; the double back-propagation does not fit the 35B MoE on our hardware (its peak already exceeded 64\,GB at 14B), which is itself part of the case for a data-free signal. All three signals, propagated, isolated, and HAWQ-V2, go through the same greedy solver at the same budget with the same candidate sets (2-bit veto copied from the propagated manifest, 205 of the 252 linear tensors), and the 72 normalization weights are pinned to the propagated allocation, so the flexible tensor set and its byte total (within $0.02\%$) are identical. The three manifests were converted and evaluated under one runtime; perplexity differences are tested per sequence with a paired $t$-test on the 128 WikiText-2 sequences.

\begin{table}[t]
\centering
\caption{Three signals through one allocator on Qwen3-8B at a 6\,GB target (252 flexible tensors, 72 norms pinned, byte totals within $0.02\%$). \emph{Upper}: parameter-weighted mean bit-width by projection; agreement is the share of flexible tensors given the same bit-width as the propagated allocation, and $\rho_4$ the Spearman correlation of the signal's 4-bit cost with the propagated one. \emph{Lower}: the three builds evaluated under one runtime (WikiText-2, 128 sequences of 2{,}048 tokens; MMLU 5-shot, full test set). ``Lower'' counts sequences on which the build beats the propagated build; $p$ is the paired $t$-test.}
\label{tab:hawqv2}
\small
\resizebox{\textwidth}{!}{%
\begin{tabular}{lrrrrrrrrr}
\toprule
Signal & $q$ & $k$ & $v$ & $o$ & gate & up & down & Agreement & $\rho_4$ \\
\midrule
Propagated probe (data-free) & 5.3 & 6.4 & 7.8 & 6.3 & 5.4 & 5.5 & 6.0 & --- & --- \\
Isolated estimator (data-free) & 6.0 & 8.0 & 8.0 & 6.0 & 5.5 & 5.4 & 5.7 & 50\% & 0.04 \\
HAWQ-V2 (calibration) & 5.6 & 6.6 & 7.6 & 5.4 & 5.6 & 5.8 & 5.6 & 41\% & 0.19 \\
\bottomrule
\end{tabular}}

\vspace{4pt}
\resizebox{0.9\textwidth}{!}{%
\begin{tabular}{lrrrrr}
\toprule
Build & Size (GB) & Median PPL & Mean PPL & Lower / $p$ vs.\ propagated & MMLU \\
\midrule
Propagated probe & 7.3 & 9.630 & 9.641 & --- & 0.749 \\
Isolated estimator & 7.3 & 9.737 & 9.698 & 20 / $<\!10^{-15}$ & 0.749 \\
HAWQ-V2 & 7.3 & 9.656 & 9.641 & 63 / 0.96 & 0.750 \\
\bottomrule
\end{tabular}}
\end{table}

All three signals agree on the coarse shape, value projections high and gate/up projections low, and disagree tensor by tensor: HAWQ-V2 matches the propagated allocation on 41\% of flexible tensors and the isolated estimator on 50\%, with 4-bit cost rank correlations of 0.19 and 0.04 against the propagated signal. HAWQ-V2 spends least on the output projection (5.4 bits against 6.3) and most on the up projection (5.8 against 5.5).

On WikiText-2 the propagated probe and HAWQ-V2 are indistinguishable: mean perplexity $9.641$ for both, medians $9.630$ and $9.656$, and per sequence the propagated build is lower on 65 of 128 sequences and the HAWQ-V2 build on 63 (paired $t$-test $p = 0.96$). The isolated estimator trails both: median $9.737$, higher than the propagated build on 108 of 128 sequences ($p < 10^{-15}$), a $0.6\%$ gap in mean perplexity, and the same margin against HAWQ-V2 (106 of 128). Read with \S\ref{sec:iso_ablation}, the two ablations agree. On the MoE model, where the routed expert stacks hold 96\% of the scored parameters and every signal puts them low, the cost signal does not change the 21\,GB build. On the dense model, where the flexible tensors are comparably sized and the allocation is decided tensor by tensor, the propagated probe recovers what the loss-curvature signal knows, without gradients or data, and the calibrated isolated estimator does not. On MMLU the three builds are indistinguishable at $0.749$, $0.749$, and $0.750$ ($\pm 0.003$): at this budget every allocation preserves the dense model's factual recall, and perplexity is the only metric that separates the signals. The HAWQ-V2 signal uses a single Hutchinson seed; the comparison therefore establishes parity of the data-free probe with one draw of the gradient estimator, not superiority over it.

\subsection{MLA Probe Coverage}
\label{sec:mla_results}

Table~\ref{tab:mla} compares probe strategies on GLM-4.7-Flash, which uses Multi-head Latent Attention with a compressed KV projection. All builds target 16\,GB and use the same allocator; the MLA-aware manifest yields a 16.2\,GB build and the two isotropic-probe manifests yield 14.6\,GB builds, so part of the gap is size.

\begin{table}[t]
\centering
\caption{MLA probe strategies on GLM-4.7-Flash (MoE + MLA), 16\,GB target. ``Propagated'' probes pass through the network without the latent stage; ``adaptive'' adds the divergence monitor of \S\ref{sec:signal}. The community uniform 4-bit build is the baseline.}
\label{tab:mla}
\begin{tabular}{lrrr}
\toprule
Strategy & Size (GB) & Median PPL & $\Delta$ vs BF16 \\
\midrule
BF16 & 58.2 & 8.470 & --- \\
Probe, MLA-aware (latent stage) & 16.2 & 8.700 & $+$2.7\% \\
Probe, adaptive & 14.6 & 9.354 & $+$10.4\% \\
Probe, propagated & 14.6 & 9.486 & $+$12.0\% \\
Uniform 4-bit & --- & 10.075 & $+$18.9\% \\
\bottomrule
\end{tabular}
\end{table}

MLA-aware probes reach $+$2.7\% degradation vs BF16, against $+$10.4\% for the best isotropic-probe manifest and $+$18.9\% for uniform 4-bit. Isotropic probes in the latent space spend energy on directions the compression path never produces, so the scores they assign to the 94 post-bottleneck weights do not reflect inference-time behaviour and the allocator misplaces bytes; latent probes restrict the measurement to the active subspace.

\subsection{Analysis Time}
\label{sec:timing}

Table~\ref{tab:timing} reports wall-clock analysis time for each model. For three models we also ran the rate-distortion alternative that scores each tensor by quantizing it at 13 (bit, group-size) configurations in isolation; the probe pass is 5 to 50 times faster.

\begin{table}[t]
\centering
\caption{RAM analysis time on the M2 Ultra. RD time is the measured wall-clock of a per-tensor rate-distortion scan at 13 configurations, where we ran it. Asterisk (*) marks lazy loading (one layer at a time).}
\label{tab:timing}
\small
\begin{tabular}{lrrrr}
\toprule
Model & BF16 Size & Probe Time & RD Time & Speedup \\
\midrule
Qwen3-8B & 15\,GB & 36\,s & 195\,s & 5.4$\times$ \\
Qwen3-30B-A3B & 61\,GB & 53\,s* & 2{,}634\,s & 50$\times$ \\
Qwen3.5-35B-A3B & 69\,GB & 188\,s & --- & --- \\
GLM-4.7-Flash & 60\,GB & 99\,s & 1{,}847\,s & 19$\times$ \\
Llama-4-Scout & 217\,GB & 214\,s* & --- & --- \\
MiniMax-M2.5 & 230\,GB & 301\,s* & --- & --- \\
Qwen3.5-122B-A10B & 250\,GB & 345\,s* & --- & --- \\
Llama-4-Maverick & 803\,GB & 539\,s* & --- & --- \\
\bottomrule
\end{tabular}
\end{table}

\paragraph{Lazy loading.}
For models exceeding available memory, RAM loads one decoder layer at a time, so peak memory is set by the largest layer plus the probe batch rather than by the model. The run reports record peak memory of about 25\,GB for the 217 to 250\,GB models and about 109\,GB for the 803\,GB Llama-4-Maverick, which the 192\,GB machine analyses in under 10 minutes.

\subsection{Median vs.\ Mean Perplexity}
\label{sec:median_ppl}

A handful of sequences can dominate corpus-level perplexity. On GLM-4.7-Flash the BF16 model has 5 outlier sequences among 128 (per-sequence perplexity above 25{,}000) that lift its mean to $11.52$ while its median is $8.470$; the uniform 4-bit build has 6 outliers, mean $14.75$, median $10.075$. The ordering is the same under both statistics, but the size of the gap under the mean is set by a few pathological sequences on which the BF16 model itself is near-random. We therefore report the median of per-sequence perplexity as the primary metric and give the mean alongside it where both exist (Table~\ref{tab:safety}).

\subsection{Downstream Task Evaluation}
\label{sec:downstream}

To check that perplexity improvements carry over to a downstream task, we evaluate MMLU (5-shot, full test set) with the \texttt{lm-evaluation-harness}~\cite{gao2024eval} on Qwen3.5-35B-A3B builds along a budget sweep, all allocated from the same probe manifest. Standard errors are those reported by the harness (about $0.35$ percentage points at the full test-set size).

\begin{table}[t]
\centering
\caption{Downstream evaluation and allocation ablation on Qwen3.5-35B-A3B. \emph{Upper}: at matched size ($\sim$34\,GB), RAM matches or beats uniform 8-bit on both metrics. \emph{Middle}: budget sweep from one probe manifest; the 25\,GB-target build is the smallest that beats uniform 4-bit on both metrics. \emph{Lower}: allocation ablation at 19.5\,GB. Config~G forces all attention tensors to 8-bit and reduces the rest to fit; Config~H forces all MLP and expert tensors to 4-bit and upgrades attention. MMLU standard errors are $0.0034$ to $0.0039$, except $0.0056$ for the 25\,GB-target build.}
\label{tab:downstream}
\small
\begin{tabular}{lrrcl}
\toprule
Build & Size (GB) & Med.\ PPL & MMLU & Note \\
\midrule
\multicolumn{5}{l}{\emph{Reference points}} \\
BF16 & 69.3 & 6.494 & 0.721 & \\
Uniform 8-bit g64 & 34.3 & 6.517 & 0.720 & \\
Uniform 4-bit g64 & 17.4 & 6.764 & 0.704 & \\
\midrule
\multicolumn{5}{l}{\emph{RAM budget sweep (one probe pass)}} \\
RAM, 34\,GB target & 33.6 & 6.504 & \textbf{0.731} & beats BF16 by 1.0\,pp \\
RAM, 25\,GB target & 22.9 & 6.584 & 0.713 & beats uniform 4-bit on both \\
RAM, 23\,GB target & 21.2 & 6.525 & 0.683 & PPL only \\
RAM, 21\,GB target & 19.5 & \textbf{6.585} & 0.657 & PPL only \\
\midrule
\multicolumn{5}{l}{\emph{Allocation ablation at 19.5\,GB}} \\
G: force attention to 8-bit & 19.5 & 6.598 & 0.645 & \\
H: force MLP/experts to 4-bit & 19.5 & 6.622 & 0.589 & \\
\bottomrule
\end{tabular}
\end{table}

\paragraph{At moderate budgets RAM wins on both metrics.} At 33.6\,GB (48\% of BF16), the RAM build scores $0.731$ on MMLU against $0.721$ for BF16 and $0.720$ for uniform 8-bit at 34.3\,GB, a $1.0$ percentage-point margin over BF16 that is about two standard errors.

\paragraph{PPL and MMLU cross over near a third of the BF16 size.} Every RAM build in the sweep beats uniform 4-bit g64 on median perplexity. On MMLU, the 19.5\,GB and 21.2\,GB builds trail uniform 4-bit ($0.657$ and $0.683$ vs $0.704$), and the 22.9\,GB build (33\% of BF16, 4.86 average bits) is the smallest that beats it ($0.713$). The mechanism is visible in the manifests: the perplexity-optimal allocation at tight budgets pushes routed expert stacks to 3-bit, which average next-token prediction tolerates because each expert fires rarely, while MMLU queries every knowledge domain and therefore every expert. From about 23\,GB upward the budget keeps the expert stacks at 4 to 5 bits.

\paragraph{Both blanket rules lose.} Config~G (force every attention tensor to 8-bit, reduce the rest) is worse than RAM on both perplexity ($6.598$ vs $6.585$, a gap inside the resolution of a 128-sequence median) and MMLU ($0.645$ vs $0.657$, $1.2$ points at $0.4$-point standard errors). Config~H (force MLP and expert tensors to 4-bit, upgrade attention) is worst on both ($6.622$; $0.589$). Taking bytes from the expert and MLP tensors costs far more factual recall than taking them from attention, and adding bytes to every attention tensor helps less than adding them to the specific tensors the probes flag.

\paragraph{Practical recommendation.} At budgets of a third of BF16 and above, the RAM build dominates uniform quantization on both metrics. Below that, RAM still gives lower perplexity, and practitioners who prioritise knowledge-intensive tasks may prefer uniform 4-bit's even distribution of precision across experts.

% ==============================================================================
\section{Analysis}
\label{sec:analysis}
% ==============================================================================

\subsection{Isolated and Propagated Sensitivities Are Different Quantities}
\label{sec:iso_vs_prop}

Proposition~\ref{prop:weighted} says the two signals estimate different operators; the manifests show how different. On Qwen3.5-35B-A3B we matched the exact isolated 4-bit sensitivity $\|\Delta\|_F^2/\|\mathbf{W}\|_F^2$ from the spectral-flatness measurement (\S\ref{app:deff_measurement}) with the propagated 4-bit divergence in the build manifest for the 390 tensors present in both, and computed the Spearman rank correlation across tensors (Table~\ref{tab:iso_vs_prop}).

\begin{table}[t]
\centering
\caption{Rank agreement between the isolated exact estimator and the propagated allocator signal, Qwen3.5-35B-A3B, 4-bit, 390 matched tensors. The two signals order tensors independently overall and in opposite directions within attention.}
\label{tab:iso_vs_prop}
\small
\begin{tabular}{lrr}
\toprule
Tensor set & $n$ & Spearman $\rho$ \\
\midrule
All matched tensors & 390 & $-0.01$ \\
Attention projections & 190 & $-0.43$ \\
Shared-expert projections & 120 & $0.06$ \\
Routed expert stacks & 40 & $0.59$ \\
Routers & 40 & $0.53$ \\
\midrule
Overlap of the two top-30 lists & & 7 of 30 \\
\bottomrule
\end{tabular}
\end{table}

The isolated estimator's 30 most sensitive tensors are 29 routers and one attention projection; the propagated signal's are 22 attention projections and 8 routers. The isolated estimator sees a $(256 \times 2{,}048)$ router as the worst-quantized tensor per unit of weight energy, which it is; the propagated signal sees that a router's rounding error barely moves the layer output under the inputs the network produces, while an attention output projection's does. Neither is wrong about what it measures. The allocation results in \S\ref{sec:experiments} were produced by the propagated signal, and \S\ref{sec:iso_ablation} reports what happens when the isolated estimator allocates instead.

\subsection{What the Probe Measures That Weight Statistics Do Not}
\label{sec:why_probes}

Existing data-free methods rank tensors by per-tensor weight statistics: kurtosis~\cite{akhondzadeh2025kurtail}, the Frobenius norm of the rounding error~\cite{zhang2025mxq}, spectral properties. Theorem~\ref{thm:main} makes the honest position clear: an \emph{isotropic} probe on an \emph{isolated} tensor estimates $\|\Delta\|_F^2$ and nothing else, so it carries the same information as the Frobenius statistic, only with a known and predictable noise level. What the allocator's signal adds is propagation. The input to layer $\ell$ is the reference output of the layers before it, so the probe covariance at each tensor is the one the network produces, and the perturbation is measured after the layer's nonlinearity and residual path. Two tensors with identical $\|\Delta\|_F^2$ can then score differently because the network feeds them different inputs or damps their errors differently. The same mechanism is what the MLA fix exploits (\S\ref{sec:mla}): passing probes through the compression path is propagation applied inside a block. We have measured the calibration of the isolated estimator (\S\ref{sec:saturation}); the propagated signal is validated only end-to-end, by the quality of the builds it produces.

\subsection{Which Tensors Are Sensitive}
\label{sec:sensitivity_patterns}

The propagated 4-bit divergence in the manifests orders tensor classes consistently across the MoE models, and differently from the ``protect attention'' folklore in one respect and from its opposite in another:
\begin{itemize}[leftmargin=*,topsep=2pt,itemsep=1pt]
    \item On Qwen3.5-35B-A3B, the 20 most sensitive tensors are 16 attention projections (output and value projections of the late full-attention layers, then the fused QKV projections of the linear-attention layers) and 4 MoE routers. Routers sit at the 90th percentile of sensitivity at the median, attention at the 65th, shared experts at the 46th, and routed expert stacks at the 34th. GLM-4.7-Flash and Qwen3-30B-A3B show the same ordering: routers and attention at the top, routed experts at the bottom.
    \item On the dense Qwen3-8B the order flips: 14 of the 20 most sensitive tensors are MLP projections, and MLP tensors sit at the 61st percentile against the 42nd for attention.
    \item The blanket rule fails on both models for the same reason. Sensitivity is a property of individual tensors, not of a class: on the 35B model 93 of 260 attention tensors earn 16-bit while 22 earn 2-bit, and the routed expert stacks that dominate the byte count spread over four bit-widths (Appendix~\ref{app:allocation}). Lifting a whole class, as Config~G does, pays for tolerant tensors; crushing a whole class, as Config~H does, hits the tensors that carry factual recall.
\end{itemize}

\subsection{MoE Expert Handling}

For Mixture-of-Experts models, expert tensors are stored as 3-D arrays $[E, d_{\text{in}}, d_{\text{out}}]$ in fused switch modules where all $E$ experts within a module share quantization parameters. RAM scores and allocates each fused (layer, projection) stack as one tensor, so its divergence is the propagated effect of quantizing all $E$ experts together, its size is the sum over experts, and the solver's output is directly implementable by the inference runtime. On Qwen3.5-35B-A3B this yields 120 expert stacks (40 layers, 3 projections) carrying 32.2 of the model's 33.6 billion scored parameters.

% ==============================================================================
\section{Limitations and Discussion}
\label{sec:limitations}
% ==============================================================================

\paragraph{Data-free vs.\ calibration-based methods.}
RAM is data-free by design and operates in a different regime from calibration-based methods (GPTQ, AWQ, LLM-MQ, SliM-LLM). On Qwen3-8B (dense), AWQ reaches $9.746$ median perplexity at $4.4$\,GB using calibration data; RAM needs $7.3$\,GB to reach $9.628$, and at the tight 4\,GB target it trails uniform 4-bit (\S\ref{sec:safety_experiment}). Calibration data captures activation-dependent sensitivity that random probes do not, and on a small dense model there is no tolerant parameter class for the allocator to exploit. Calibration-based methods in turn need a representative dataset, full-model forward passes, and hours of compute, whereas RAM completes in seconds to minutes on any model without data access. A hybrid, with RAM for the initial allocation and calibration-based refinement for critical tensors, is a natural next step.

\paragraph{RTN base quantizer.}
RAM uses round-to-nearest (RTN) quantization. On 30 tensors of Qwen3.5-35B-A3B at 4-bit, HQQ~\cite{badri2024hqq} lowers per-tensor NRMSE by $2.96\%$ on average and a random Hadamard rotation~\cite{chee2024quip} by $7.92\%$, while a data-free Hessian proxy for OBQ~\cite{frantar2023gptq} and a weight-only AdaRound~\cite{nagel2020adaround} objective do not help (Appendix~\ref{app:alt_methods}). At the model level we have one data point: on Qwen3-8B at uniform 4-bit g64, HQQ-optimized weights evaluated in BF16 reach $9.768$ median perplexity against $9.987$ for RTN, a $2.2\%$ improvement that our packed-weight conversion did not preserve. Combining a better per-tensor quantizer with mixed-precision allocation remains open, and we make no claim that RTN is optimal.

\paragraph{Isotropic probes and the propagated signal.}
Theorem~\ref{thm:main} and the flatness measurement of \S\ref{sec:saturation} concern the isolated estimator with $\mathbf{x} \sim \mathcal{N}(\mathbf{0}, \mathbf{I})$. Proposition~\ref{prop:weighted} extends the estimator to the propagated setting, but its variance depends on the network's own input second moment and downstream Jacobian, and \S\ref{sec:propagated_calibration} shows that on one of three models the propagated signal is not calibrated per sequence. Its ranking is nevertheless stable at the allocator's probe count on all three. The Gaussian probe distribution is a modelling choice: the MLA extension shows that architecture-specific probe shaping recovers signal when the covariance structure is known, and a general treatment using estimated input statistics from a small unlabelled corpus would bridge the data-free and calibration-based regimes.

\paragraph{The estimate of the objective is good; the objective is the wrong target.} Section~\ref{sec:objective} shows that the propagated quadratic form estimates the calibration objective far better than the block-output cosine the reported builds were allocated with, and that allocating from that objective, with real data or without, loses to the block-output cosine at matched bytes on Qwen3.8-27B. The measurement covers three non-MoE models (two of them hybrids) and the allocation test one model and one quantizer family; the MoE expert stacks, whose inputs are routed, were not measured. The downstream-weighted quadratic form, Proposition~\ref{prop:weighted} with $\mathbf{M}$ the rest of the block, is the signal that would combine the estimate's cost with the cosine's allocation quality, and it is untested. On the hybrid 27B the adaptation rule must be disabled for the estimate to hold; with it off the three models agree.

\paragraph{Rank swaps and end-to-end quality.}
Section~\ref{sec:saturation_experiment} shows that residual rank disagreements at 20 probes are between tensors with near-equal isolated sensitivity. Equal objective cost does not imply equal end-to-end quality: two allocations with the same total divergence can differ in perplexity. We have not measured the perplexity spread across probe seeds for the reported builds, and we do not claim allocation stability beyond the objective.

\paragraph{Dense model gains.}
On the single dense model tested (Qwen3-8B) the picture is mixed: RAM matches BF16 at 48\% of its size but loses to uniform 4-bit at a tight budget. MoE models have high per-tensor sensitivity heterogeneity, which gives the allocator room to work; dense models have more uniform sensitivity and a smaller vocabulary-relative budget for the unquantized embedding tables.

\paragraph{Downstream evaluation.}
We evaluate on WikiText-2 perplexity and MMLU on one model family. Perplexity and MMLU within a few percent of BF16 do not certify behaviour on procedural or agentic tasks: in companion work we found that compressed models that pass such guards can still invent steps in multi-step execution~\cite{kennedy2026fidelity}. A broader multi-task evaluation, and a behavioural one, would strengthen the quality claims here.

% ==============================================================================
\section{Conclusion}
\label{sec:conclusion}
% ==============================================================================

We presented RAM, a data-free mixed-precision quantization framework built on a calibrated estimator. Gaussian probes provide unbiased estimates of the squared Frobenius norm of quantization error (Theorem~\ref{thm:main}), and because quantization error is spectrally flat, with measured effective dimensionality within $4$ to $7\%$ of the i.i.d.-noise value $mn/(m{+}n)$, a single probe measures per-tensor sensitivity to $4$ to $7\%$ with error bars predictable from tensor shape alone. Propagating the probes through the network applies the same estimator to the input- and downstream-weighted error operator, which is the calibration objective of GPTQ without data, and a multiple-choice knapsack solver converts one probe pass into a build at any byte budget. The two signals are different quantities: they rank tensors independently, the propagated one is calibrated per sequence on two of three models and not on the third, and its ranking is stable at fifty sequences on all three because it spans ten times the range across tensors. The propagated form is a data-free estimate of the calibration objective itself: its per-tensor values rank-correlate 0.8 with the objective under real activations, where the isolated estimator does not correlate at all. Through one allocator at matched bytes, the calibrated isolated estimator allocates at least as well as the propagated signal on the MoE model and $0.6\%$ worse in perplexity on the dense model, where the propagated probe ties gradient-based HAWQ-V2 without gradients or data.

Across seven architectures, RAM builds beat size-comparable uniform 4-bit builds by 3.5 to 13.6\% in median perplexity on the MoE models, beat a GPTQ-style heuristic at a smaller size, and reach BF16-level MMLU at half the BF16 size, with no calibration data, no gradients, and no GPU cluster. The analysis pass scales to an 803\,GB model in nine minutes on one workstation.

Two findings reach beyond RAM. First, sensitivity is a property of tensors, not of tensor classes: the probes rank attention output projections and MoE routers most sensitive and routed experts least, and both ``protect attention'' and ``deprotect MLP'' rules lose to per-tensor allocation. Second, median per-sequence perplexity is the safer headline statistic; a few outlier sequences on which even the BF16 model is near-random can dominate the mean.

% ==============================================================================
% References
% ==============================================================================

\bibliographystyle{plain}
\bibliography{references}

@inproceedings{frantar2023gptq,
  title     = {{GPTQ}: Accurate Post-Training Quantization for Generative Pre-trained Transformers},
  author    = {Frantar, Elias and Ashkboos, Saleh and Hoefler, Torsten and Alistarh, Dan},
  booktitle = {International Conference on Learning Representations (ICLR)},
  year      = {2023},
  url       = {https://arxiv.org/abs/2210.17323}
}

@inproceedings{lin2024awq,
  title     = {{AWQ}: Activation-aware Weight Quantization for On-Device {LLM} Compression and Acceleration},
  author    = {Lin, Ji and Tang, Jiaming and Tang, Haotian and Yang, Shang and Chen, Wei-Ming and Wang, Wei-Chen and Xiao, Guangxuan and Dang, Xingyu and Gan, Chuang and Han, Song},
  booktitle = {Conference on Machine Learning and Systems (MLSys)},
  year      = {2024},
  url       = {https://arxiv.org/abs/2306.00978}
}

@inproceedings{kim2024squeezellm,
  title     = {{SqueezeLLM}: Dense-and-Sparse Quantization},
  author    = {Kim, Sehoon and Hooper, Coleman and Gholami, Amir and Dong, Zhen and Li, Xiuyu and Shen, Sheng and Mahoney, Michael W. and Keutzer, Kurt},
  booktitle = {International Conference on Machine Learning (ICML)},
  year      = {2024},
  url       = {https://arxiv.org/abs/2306.07629}
}

@inproceedings{dettmers2024spqr,
  title     = {{SpQR}: A Sparse-Quantized Representation for Near-Lossless {LLM} Weight Compression},
  author    = {Dettmers, Tim and Svirschevski, Ruslan and Egiazarian, Vage and Kuznedelev, Denis and Frantar, Elias and Ashkboos, Saleh and Borzunov, Alexander and Hoefler, Torsten and Alistarh, Dan},
  booktitle = {International Conference on Learning Representations (ICLR)},
  year      = {2024},
  url       = {https://arxiv.org/abs/2306.03078}
}

@inproceedings{chee2024quip,
  title     = {{QuIP}: 2-Bit Quantization of Large Language Models with Guarantees},
  author    = {Chee, Jerry and Cai, Yaohui and Kuleshov, Volodymyr and De Sa, Christopher},
  booktitle = {Advances in Neural Information Processing Systems (NeurIPS)},
  year      = {2024},
  url       = {https://arxiv.org/abs/2307.13304}
}

@article{egiazarian2024aqlm,
  title     = {Extreme Compression of Large Language Models via Additive Quantization},
  author    = {Egiazarian, Vage and Panferov, Andrei and Kuznedelev, Denis and Frantar, Elias and Babenko, Artem and Alistarh, Dan},
  journal   = {arXiv preprint arXiv:2401.06118},
  year      = {2024},
  url       = {https://arxiv.org/abs/2401.06118}
}

@inproceedings{nagel2020adaround,
  title     = {Up or Down? {Adaptive} Rounding for Post-Training Quantization},
  author    = {Nagel, Markus and Amjad, Rana Ali and van Baalen, Mart and Louizos, Christos and Blankevoort, Tijmen},
  booktitle = {International Conference on Machine Learning (ICML)},
  year      = {2020},
  url       = {https://arxiv.org/abs/2004.10568}
}

@inproceedings{tang2023easyquant,
  title     = {{EasyQuant}: An Efficient Data-free Quantization Algorithm for {LLMs}},
  author    = {Tang, Hanlin and Sun, Yifu and Wu, Decheng and Liu, Kai and Zhu, Jianchen and Kang, Zhanhui},
  booktitle = {Proceedings of the 2023 Conference on Empirical Methods in Natural Language Processing (EMNLP)},
  year      = {2023},
  doi       = {10.18653/v1/2023.emnlp-main.565},
  url       = {https://doi.org/10.18653/v1/2023.emnlp-main.565}
}

@inproceedings{zhang2025mxq,
  title     = {A Mixed Quantization Approach for Data-Free Quantization of {LLMs}},
  author    = {Zhang, Feng and Liu, Yanbin and Li, Weihua and Wang, Xiaodan and Bai, Quan},
  booktitle = {Proceedings of the 17th International Conference on Agents and Artificial Intelligence (ICAART)},
  pages     = {353--363},
  publisher = {SCITEPRESS},
  year      = {2025},
  doi       = {10.5220/0013159100003890},
  url       = {https://doi.org/10.5220/0013159100003890}
}

@misc{badri2024hqq,
  title     = {Half-Quadratic Quantization of Large Machine Learning Models},
  author    = {Badri, Hicham and Shaji, Appu},
  howpublished = {Mobius Labs, \url{https://github.com/mobiusml/hqq}},
  year      = {2024}
}

@article{malinovskii2024higgs,
  title     = {Pushing the Limits of Large Language Model Quantization via the Linearity Theorem},
  author    = {Malinovskii, Vladimir and Panferov, Andrei and Ilin, Ivan and Guo, Han and Richt{\'a}rik, Peter and Alistarh, Dan},
  journal   = {arXiv preprint arXiv:2411.17525},
  year      = {2024},
  url       = {https://arxiv.org/abs/2411.17525}
}

@article{akhondzadeh2025kurtail,
  title     = {{KurTail}: Kurtosis-based {LLM} Quantization},
  author    = {Akhondzadeh, Mohammad Sadegh and Bojchevski, Aleksandar and Eleftheriou, Evangelos and Dazzi, Martino},
  journal   = {arXiv preprint arXiv:2503.01483},
  year      = {2025},
  url       = {https://arxiv.org/abs/2503.01483}
}

@article{lee2021datafreemp,
  title     = {Data-free Mixed-precision Quantization Using Novel Sensitivity Metric},
  author    = {Lee, Donghyun and Cho, Minkyoung and Lee, Seungwon and Song, Joonho and Choi, Changkyu},
  journal   = {arXiv preprint arXiv:2103.10051},
  year      = {2021},
  url       = {https://arxiv.org/abs/2103.10051}
}

@article{zhang2026nsds,
  title     = {Beyond Outliers: A Data-Free Layer-wise Mixed-Precision Quantization Approach Driven by Numerical and Structural Dual-Sensitivity},
  author    = {Zhang, Hengyuan and Chen, Xinrong and Su, Zunhai and Liang, Xiao and Xiong, Jing and Xu, Wendong and others},
  journal   = {arXiv preprint arXiv:2603.17354},
  year      = {2026},
  url       = {https://arxiv.org/abs/2603.17354}
}

@article{yang2026alphaq,
  title     = {{AlphaQ}: Calibration-Free Bit Allocation for Mixture-of-Experts Quantization},
  author    = {Yang, Wanqi and Ma, Yuexiao and Conzelmann, Alexander and Zheng, Xiawu and Mahoney, Michael W. and Rusch, T. Konstantin and others},
  journal   = {arXiv preprint arXiv:2606.04980},
  year      = {2026},
  url       = {https://arxiv.org/abs/2606.04980}
}

@inproceedings{li2023llmmq,
  title     = {{LLM-MQ}: Mixed-precision Quantization for Efficient {LLM} Deployment},
  author    = {Li, Shiyao and Ning, Xuefei and Hong, Ke and Liu, Tengxuan and Wang, Luning and Li, Xiuhong and Zhong, Kai and Dai, Guohao and Yang, Huazhong and Wang, Yu},
  booktitle = {NeurIPS 2023 Workshop on Efficient Natural Language and Speech Processing (ENLSP)},
  year      = {2023},
  url       = {https://neurips2023-enlsp.github.io/papers/paper_4.pdf}
}

@article{huang2024slimlm,
  title     = {{SliM-LLM}: Salience-Driven Mixed-Precision Quantization for Large Language Models},
  author    = {Huang, Wei and Qin, Haotong and Liu, Yangdong and Li, Yawei and Liu, Xianglong and Benini, Luca and Magno, Michele and Qi, Xiaojuan},
  journal   = {arXiv preprint arXiv:2405.14917},
  year      = {2024},
  url       = {https://arxiv.org/abs/2405.14917}
}

@article{cui2024cherryq,
  title     = {Cherry on Top: Parameter Heterogeneity and Quantization in Large Language Models},
  author    = {Cui, Wanyun and Wang, Qianle},
  journal   = {arXiv preprint arXiv:2404.02837},
  year      = {2024},
  url       = {https://arxiv.org/abs/2404.02837}
}

@article{huang2024mcmoe,
  title     = {Mixture Compressor for Mixture-of-Experts {LLMs} Gains More},
  author    = {Huang, Wei and Liao, Yue and Liu, Jianhui and He, Ruifei and Tan, Haoru and Zhang, Shiming and Li, Hongsheng and Liu, Si and Qi, Xiaojuan},
  journal   = {arXiv preprint arXiv:2410.06270},
  year      = {2024},
  url       = {https://arxiv.org/abs/2410.06270}
}

@article{yao2026gamma,
  title     = {{GAMMA}: Global Bit Allocation for Mixed-Precision Models under Arbitrary Budgets},
  author    = {Yao, Zhangyang and Zhao, Haiyan and Wang, Haoyu and Han, Xu},
  journal   = {arXiv preprint arXiv:2605.18475},
  year      = {2026},
  url       = {https://arxiv.org/abs/2605.18475}
}

@article{misra2026mixquant,
  title     = {{MixQuant}: Adaptive Mixed-Precision Quantization for Large Language Models},
  author    = {Misra, Ashitabh and Agrawal, Madhav and Jain, Arham and Abdelzaher, Tarek},
  journal   = {arXiv preprint arXiv:2607.23047},
  year      = {2026},
  url       = {https://arxiv.org/abs/2607.23047}
}

@article{yoshida2026casa,
  title     = {Beyond Scalar Sensitivity: Activation-Aware Mixed-Precision {LLM} Quantization with Cross-Layer Refinement},
  author    = {Yoshida, Akihiro and Ichikawa, Yuma},
  journal   = {arXiv preprint arXiv:2609.25916},
  year      = {2026},
  url       = {https://arxiv.org/abs/2609.25916}
}

@article{zhao2026bitsmoe,
  title     = {{BitsMoE}: Efficient Spectral Energy-Guided Bit Allocation for {MoE} {LLM} Quantization},
  author    = {Zhao, Jiayu and Teng, Zihan and Fan, Minhao and Ma, Tianrui and Ren, Wentao and Chen, Song and others},
  journal   = {arXiv preprint arXiv:2606.00079},
  year      = {2026},
  url       = {https://arxiv.org/abs/2606.00079}
}

@article{lee2026qstrata,
  title     = {{Q-Strata}: Hierarchical Bit Allocation for Mixed-Precision Quantization of Mixture-of-Experts {LLMs}},
  author    = {Lee, Deokjae and Chu, Sihun and Song, Hyun Oh},
  journal   = {arXiv preprint arXiv:2608.30564},
  year      = {2026},
  url       = {https://arxiv.org/abs/2608.30564}
}

@inproceedings{dong2019hawq,
  title   = {{HAWQ}: {Hessian} {AW}are Quantization of Neural Networks With Mixed-Precision},
  author  = {Dong, Zhen and Yao, Zhewei and Gholami, Amir and Mahoney, Michael W. and Keutzer, Kurt},
  booktitle = {International Conference on Computer Vision (ICCV)},
  year    = {2019},
  doi     = {10.1109/ICCV.2019.00038},
  url     = {https://openalex.org/W2982041622}
}

@inproceedings{dong2020hawqv2,
  title   = {{HAWQ-V2}: {Hessian} Aware trace-Weighted Quantization of Neural Networks},
  author  = {Dong, Zhen and Yao, Zhewei and Cai, Yaohui and Arfeen, Daiyaan and Gholami, Amir and Mahoney, Michael W. and Keutzer, Kurt},
  booktitle = {Advances in Neural Information Processing Systems (NeurIPS)},
  year    = {2020},
  url     = {https://arxiv.org/abs/1911.03852}
}

@article{zhang2024lqer,
  title     = {{LQER}: Low-Rank Quantization Error Reconstruction for {LLMs}},
  author    = {Zhang, Cheng and Cheng, Jianyi and Constantinides, George A. and Zhao, Yiren},
  journal   = {arXiv preprint arXiv:2402.02446},
  year      = {2024},
  url       = {https://arxiv.org/abs/2402.02446}
}

@article{hutchinson1989stochastic,
  title     = {A Stochastic Estimator of the Trace of the Influence Matrix for {Laplacian} Smoothing Splines},
  author    = {Hutchinson, M. F.},
  journal   = {Communications in Statistics -- Simulation and Computation},
  volume    = {18},
  number    = {3},
  pages     = {1059--1076},
  year      = {1989},
  doi       = {10.1080/03610918908812806},
  url       = {https://doi.org/10.1080/03610918908812806}
}

@article{bekas2007estimator,
  title     = {An Estimator for the Diagonal of a Matrix},
  author    = {Bekas, C. and Kokiopoulou, E. and Saad, Y.},
  journal   = {Applied Numerical Mathematics},
  volume    = {57},
  number    = {11--12},
  pages     = {1214--1229},
  year      = {2007},
  doi       = {10.1016/j.apnum.2007.01.003},
  url       = {https://doi.org/10.1016/j.apnum.2007.01.003}
}

@article{avron2011randomized,
  title     = {Randomized Algorithms for Estimating the Trace of an Implicit Symmetric Positive Semi-definite Matrix},
  author    = {Avron, Haim and Toledo, Sivan},
  journal   = {Journal of the ACM},
  volume    = {58},
  number    = {2},
  pages     = {1--34},
  year      = {2011},
  doi       = {10.1145/1944345.1944349},
  url       = {https://doi.org/10.1145/1944345.1944349}
}

@incollection{johnson1984extensions,
  title     = {Extensions of {Lipschitz} mappings into a {Hilbert} space},
  author    = {Johnson, William B. and Lindenstrauss, Joram},
  booktitle = {Contemporary Mathematics},
  publisher = {American Mathematical Society},
  volume    = {26},
  pages     = {189--206},
  year      = {1984},
  doi       = {10.1090/conm/026/737400},
  url       = {https://openalex.org/W2979473749}
}

@article{amari1998natural,
  title     = {Natural Gradient Works Efficiently in Learning},
  author    = {Amari, Shun-ichi},
  journal   = {Neural Computation},
  volume    = {10},
  number    = {2},
  pages     = {251--276},
  year      = {1998},
  doi       = {10.1162/089976698300017746},
  url       = {https://doi.org/10.1162/089976698300017746}
}

@article{deepseekai2024deepseekv2,
  title     = {{DeepSeek-V2}: A Strong, Economical, and Efficient Mixture-of-Experts Language Model},
  author    = {{DeepSeek-AI}},
  journal   = {arXiv preprint arXiv:2405.04434},
  year      = {2024},
  url       = {https://arxiv.org/abs/2405.04434}
}

@misc{apple2023mlx,
  title     = {{MLX}: An Array Framework for {Apple} Silicon},
  author    = {Hannun, Awni and Digani, Jagrit and Katharopoulos, Angelos and Collobert, Ronan},
  howpublished = {Apple Machine Learning Research, \url{https://github.com/ml-explore/mlx}},
  year      = {2023}
}

@misc{gao2024eval,
  title     = {{EleutherAI/lm-evaluation-harness}: v0.4.3},
  author    = {Sutawika, Lintang and Schoelkopf, Hailey and Gao, Leo and Abbasi, Baber and Biderman, Stella and Tow, Jonathan and others},
  howpublished = {Zenodo},
  year      = {2024},
  doi       = {10.5281/zenodo.12608602},
  url       = {https://zenodo.org/doi/10.5281/zenodo.12608602}
}

@misc{llama4,
  title     = {The {Llama} 4 Herd: The Beginning of a New Era of Natively Multimodal {AI} Innovation},
  author    = {{Meta AI}},
  howpublished = {\url{https://ai.meta.com/blog/llama-4-multimodal-intelligence/}},
  year      = {2025}
}

@misc{qwen2025qwen35,
  title     = {{Qwen3.5-35B-A3B} model card},
  author    = {{Qwen Team}},
  howpublished = {Hugging Face, \url{https://huggingface.co/Qwen/Qwen3.5-35B-A3B}},
  year      = {2025}
}

@article{kennedy2026fidelity,
  title     = {Fidelity Is Not Safety: Gently-Compressed {LLMs} Pass Every Data-Free Quality Guard Yet Invent Procedure Steps in Agentic Execution},
  author    = {Kennedy, I. and Kennedy, T.},
  journal   = {arXiv preprint arXiv:2607.28196},
  year      = {2026},
  url       = {https://arxiv.org/abs/2607.28196}
}

@inproceedings{kornblith2019cka,
  title   = {Similarity of Neural Network Representations Revisited},
  author  = {Kornblith, Simon and Norouzi, Mohammad and Lee, Honglak and Hinton, Geoffrey},
  booktitle = {Proceedings of the 36th International Conference on Machine Learning (ICML)},
  year    = {2019},
  url     = {https://arxiv.org/abs/1905.00414}
}

% ==============================================================================
% Appendix
% ==============================================================================

\newpage
\appendix

\section{Proof Details and Variance Bounds}
\label{app:proofs}

\subsection{Variance of the Probe Estimator}

\begin{lemma}[Variance bound]
\label{lem:variance}
For $\mathbf{x} \sim \mathcal{N}(\mathbf{0}, \mathbf{I}_n)$ and symmetric positive semidefinite $A = \Delta^\top\Delta$:
\begin{equation}
    \mathrm{Var}[\mathbf{x}^\top A \mathbf{x}] = 2\|A\|_F^2 = 2\sum_{i,j} A_{ij}^2 .
\end{equation}
\end{lemma}

\begin{proof}
Using the fourth-moment identity for Gaussians:
\begin{align}
    \mathbb{E}[(\mathbf{x}^\top A \mathbf{x})^2] &= \mathbb{E}\!\left[\sum_{i,j,k,l} A_{ij} A_{kl} x_i x_j x_k x_l\right] \\
    &= \sum_{i,j,k,l} A_{ij} A_{kl}\, \mathbb{E}[x_i x_j x_k x_l] .
\end{align}
For Gaussian variables, $\mathbb{E}[x_i x_j x_k x_l] = \delta_{ij}\delta_{kl} + \delta_{ik}\delta_{jl} + \delta_{il}\delta_{jk}$. Therefore
\begin{align}
    \mathbb{E}[(\mathbf{x}^\top A \mathbf{x})^2] &= \sum_{i,j,k,l} A_{ij} A_{kl}(\delta_{ij}\delta_{kl} + \delta_{ik}\delta_{jl} + \delta_{il}\delta_{jk}) \\
    &= \mathrm{Tr}(A)^2 + 2\mathrm{Tr}(A^2) = \mathrm{Tr}(A)^2 + 2\|A\|_F^2 .
\end{align}
Since $\mathrm{Var}[\mathbf{x}^\top A \mathbf{x}] = \mathbb{E}[(\mathbf{x}^\top A \mathbf{x})^2] - (\mathbb{E}[\mathbf{x}^\top A \mathbf{x}])^2 = \mathrm{Tr}(A)^2 + 2\|A\|_F^2 - \mathrm{Tr}(A)^2 = 2\|A\|_F^2$.
\end{proof}

\begin{corollary}[Relative standard deviation]
The coefficient of variation of a single-probe estimate is
\begin{equation}
    \text{CV} = \frac{\sqrt{2\|A\|_F^2}}{\mathrm{Tr}(A)} = \sqrt{\frac{2}{d_{\mathrm{eff}}(A)}} .
\end{equation}
For $p$ probes, $\text{CV}_p = \text{CV}/\sqrt{p}$. With the measured median $d_{\mathrm{eff}} = 386$ on Qwen3.5-35B-A3B, a single probe gives $\text{CV}_1 = \sqrt{2/386} \approx 7.2\%$ and $p = 20$ gives $\text{CV}_{20} \approx 1.6\%$; on Qwen3.5-9B (median $d_{\mathrm{eff}} = 1{,}476$), $\text{CV}_1 \approx 3.7\%$ and $\text{CV}_{20} \approx 0.8\%$. These predictions match the empirical CVs with median ratio $0.999$ and $1.000$ (Table~\ref{tab:deff}).
\end{corollary}

\subsection{Concentration and the Role of Spectral Flatness}

Concentration bounds for quadratic forms (Hanson--Wright; see also Avron and Toledo~\cite{avron2011randomized}) control the deviation of $\|\Delta\mathbf{x}\|^2/\|\Delta\|_F^2$ in terms of the spectrum of $A = \Delta^\top\Delta$, with the effective dimensionality playing the role of the sample size:
\begin{equation}
    \Pr\!\left[\left|\frac{\|\Delta \mathbf{x}\|^2}{\|\Delta\|_F^2} - 1\right| > \varepsilon\right] \leq 2\exp\!\left(-c\, d_{\mathrm{eff}}(A)\, \min(\varepsilon, \varepsilon^2)\right)
\end{equation}
for an absolute constant $c$. A Johnson--Lindenstrauss-style bound in the ambient dimension would be valid only for an exactly isotropic error; the honest parameter is $d_{\mathrm{eff}}$, and the measurement in \S\ref{app:deff_measurement} shows $d_{\mathrm{eff}}$ is within a few percent of its isotropic-noise ceiling. For a typical $d_{\mathrm{eff}} \approx 400$ to $1{,}500$, a single probe gives a $(1 \pm 0.1)$-approximation with high probability, matching the observed per-probe CVs of $4$ to $7\%$.

\subsection{Effective Dimensionality of Quantization Error: Measured}

An earlier version of this work conjectured $d_{\mathrm{eff}} \ll \min(m,n)$ (spatially correlated, outlier-dominated error) and inferred $d_{\mathrm{eff}} \approx 30$ to $50$ from an apparent faster-than-$\sqrt{p}$ convergence slope. Direct measurement (\S\ref{app:deff_measurement}) refutes both: the convergence slope artifact is explained in \S\ref{sec:saturation_experiment}, and the true effective dimensionality is close to the \emph{maximal} value attained by an i.i.d.\ noise matrix of the same shape,
\begin{equation}
    d_{\mathrm{eff}}(\Delta^\top\Delta) \approx (0.93\text{ to }0.96) \cdot \frac{mn}{m+n},
\end{equation}
where $mn/(m{+}n)$ is the expected participation ratio of a Marchenko--Pastur spectrum with aspect ratio $n/m$ (it equals $n/2$ for square matrices and approaches $n$ for very tall ones). Group-wise RTN error behaves like white noise because within each 64-weight group the rounding residuals are effectively independent sub-step perturbations; outliers and inter-group correlation shave only a few percent off the noise ceiling. The deficit is slightly larger for MoE expert stacks quantized with shared group parameters (median ratio $0.93$) than for dense tensors ($0.96$), and a small tail of structured tensors (the $32$-row linear-attention projections, one vision-tower output projection) fall well below the ceiling. These are the tensors where more probes are needed; because the multi-probe sample variance is a free empirical estimate of each tensor's CV (Remark~\ref{rem:why_probes}), they are detected at probe time at no additional cost.

\section{Alternative Per-Tensor Quantizers}
\label{app:alt_methods}

We evaluated alternatives to RTN as the per-tensor quantizer on 30 tensors of Qwen3.5-35B-A3B at 4-bit with group size 64, measuring per-tensor NRMSE.

\paragraph{HQQ (Half-Quadratic Quantization).}
HQQ~\cite{badri2024hqq} lowers per-tensor NRMSE by $2.96\%$ on average (30 of 30 tensors improved). At the model level, on Qwen3-8B at uniform 4-bit g64, HQQ-optimized weights evaluated in BF16 reach $9.768$ median perplexity against $9.987$ for RTN. Re-packing the HQQ weights into the runtime's 4-bit format re-rounds them and loses the gain ($10.338$), so integrating HQQ into the pipeline requires a format-aware conversion that we have not built.

\paragraph{Hadamard rotation.}
A random orthogonal rotation lowers NRMSE by $7.92\%$ on average (30 of 30 tensors improved), the largest per-tensor improvement of any method tested. Deployment requires custom dequantization kernels ($y = W_q(H^\top x)$) that our runtime does not provide, so we did not evaluate it at the model level.

\paragraph{Data-free OBQ.}
Approximating the Hessian as $H \approx W^\top W$ \emph{worsens} reconstruction ($+$0.50\% NRMSE, 0 of 18 tensors improved). The weight-only Hessian is a poor proxy for the activation Hessian; OBQ needs calibration data.

\paragraph{Data-free AdaRound.}
Replacing AdaRound's~\cite{nagel2020adaround} activation-weighted objective with weight-only MSE changes NRMSE by exactly $0.000\%$: without activation data the optimal rounding direction for each element is round-to-nearest, which RTN already computes.

Where bits are allocated matters more, in our experiments, than how each tensor is rounded, but the HQQ data point shows that a better rounder can add a further percent or two at the model level.

\section{Allocation by Tensor Class}
\label{app:allocation}

Table~\ref{tab:allocation_detail} summarises the manifests behind the Qwen3.5-35B-A3B (21\,GB target), GLM-4.7-Flash (16\,GB target, MLA-aware), and Qwen3-8B (6\,GB target) builds of Table~\ref{tab:main}. ``Attention'' includes all attention and linear-attention projections; ``Routed experts'' are the fused per-layer expert stacks; ``Other'' covers small convolution and gating tensors above the 1{,}024-element threshold.

\begin{table}[h]
\centering
\caption{RAM allocation by tensor class, read from the build manifests. Counts are tensors at each bit-width; the last column is the parameter-weighted mean bit-width of the class. Embeddings, language-model heads, and norms are unquantized and not listed.}
\label{tab:allocation_detail}
\small
\resizebox{\textwidth}{!}{%
\begin{tabular}{llrrrrrrrrr}
\toprule
Model & Class & Tensors & Params & 2b & 3b & 4b & 5b & 6b & 8b/16b & Avg bits \\
\midrule
\multirow{5}{*}{Qwen3.5-35B-A3B} & Attention & 260 & 1.28B & 22 & 0 & 0 & 3 & 13 & 129/93 & 8.5 \\
 & Router & 40 & 0.02B & 0 & 0 & 0 & 0 & 0 & 0/40 & 16.0 \\
 & Shared expert & 160 & 0.13B & 0 & 0 & 0 & 0 & 3 & 139/18 & 8.0 \\
 & Routed experts & 120 & 32.2B & 3 & 31 & 74 & 12 & 0 & 0/0 & 3.8 \\
 & Other & 40 & $<$0.01B & 0 & 0 & 0 & 0 & 0 & 21/19 & 11.8 \\
\midrule
\multirow{5}{*}{GLM-4.7-Flash} & Attention & 329 & 1.02B & 0 & 5 & 12 & 17 & 26 & 131/138 & 9.0 \\
 & Router & 46 & 0.01B & 0 & 0 & 0 & 0 & 0 & 0/46 & 16.0 \\
 & Shared expert & 138 & 0.43B & 0 & 5 & 9 & 12 & 20 & 15/77 & 11.5 \\
 & Routed experts & 138 & 27.8B & 0 & 58 & 80 & 0 & 0 & 0/0 & 3.6 \\
 & Dense MLP + other & 50 & 0.06B & 0 & 0 & 1 & 0 & 0 & 23/26 & --- \\
\midrule
\multirow{3}{*}{Qwen3-8B} & Attention & 180 & 1.51B & 0 & 0 & 7 & 27 & 57 & 72/17 & 6.0 \\
 & Dense MLP & 108 & 5.44B & 0 & 0 & 0 & 68 & 26 & 14/0 & 5.6 \\
 & Other & 36 & $<$0.01B & 0 & 0 & 0 & 0 & 0 & 16/20 & 12.4 \\
\bottomrule
\end{tabular}}
\end{table}

\section{Evaluation Protocol}
\label{app:eval}

\paragraph{WikiText-2 evaluation.}
All perplexity evaluations use the WikiText-2 test split with 128 evaluation sequences of 2{,}048 tokens, seed 42, and no sliding window (each sequence is evaluated independently), except the two community uniform builds marked in Table~\ref{tab:main}, which were evaluated on 256 sequences. Our evaluator reports both mean and median perplexity.

\paragraph{Median perplexity.}
For each sequence we compute the per-sequence perplexity as $\exp(\text{mean cross-entropy loss})$. The reported median is the 50th percentile of these values. The mean is the standard corpus-level perplexity (exponential of the average cross-entropy across all tokens).

\paragraph{Hardware.}
All experiments use a single Apple Mac Studio with M2 Ultra (192\,GB unified memory, 24-core CPU, 76-core GPU). No multi-GPU or distributed setup is required.

\paragraph{Quantization.}
All quantization uses group-wise round-to-nearest with affine scaling, group size 64 unless otherwise specified, and packed integer storage with a 16-bit scale and bias per group.

\section{Direct Measurement of Quantization-Error Spectra}
\label{app:deff_measurement}

\paragraph{Protocol.}
For every weight tensor with both dimensions $\geq 32$ (embeddings and LM head excluded; 3-D fused expert stacks $(E, d_{\text{out}}, d_{\text{in}})$ reshaped to $(E\,d_{\text{out}}, d_{\text{in}})$), we compute the RTN error $\Delta = \mathbf{W} - Q(\mathbf{W}; b)$ at $b \in \{2,3,4,8\}$ (group 64) in float32 and evaluate, in closed form via the small-side Gram matrix, $S = \|\Delta\|_F^2$, $Q = \|\Delta^\top\Delta\|_F^2$, and $d_{\mathrm{eff}} = S^2/Q$. In the same pass we draw 200 independent Gaussian probes per tensor and record each single-probe estimate $\|\Delta\mathbf{x}_j\|^2$. The full measurement takes 127\,s for Qwen3.5-35B-A3B (1{,}317 tensors) and 41\,s for Qwen3.5-9B (366 tensors) on a single Apple Silicon machine.

\begin{table}[t]
\centering
\caption{Measured effective dimensionality of quantization error at 4-bit. $\widehat{d}_{\mathrm{eff}} = mn/(m{+}n)$ is the expected participation ratio of an i.i.d.\ noise matrix of the same shape. CV$_{\text{pred}} = \sqrt{2/d_{\mathrm{eff}}}$; CV$_{\text{emp}}$ is the sample CV over 200 independent probes. Unbiasedness: median ratio of the 200-probe mean to the exact $\|\Delta\|_F^2$.}
\label{tab:deff}
\small
\begin{tabular}{lrrrrr}
\toprule
& \multicolumn{3}{c}{$d_{\mathrm{eff}}$} & \multicolumn{2}{c}{Probe calibration (median)} \\
\cmidrule(lr){2-4}\cmidrule(lr){5-6}
Model & median & IQR & $d_{\mathrm{eff}}/\widehat{d}_{\mathrm{eff}}$ med.\ [IQR] & CV$_{\text{emp}}$/CV$_{\text{pred}}$ & $\bar{s}_{200}/\|\Delta\|_F^2$ \\
\midrule
Qwen3.5-35B-A3B & 386 & [354, 405] & 0.93 [0.86, 0.97] & 0.999 & 0.9998 \\
Qwen3.5-9B & 1{,}476 & [722, 2{,}709] & 0.96 [0.90, 0.98] & 1.000 & 1.0000 \\
\bottomrule
\end{tabular}
\end{table}

\paragraph{By tensor type (Qwen3.5-35B-A3B, 4-bit, medians over layers).}
Fused MoE expert stacks $(262{,}144 \times 2{,}048)$: $d_{\mathrm{eff}} = 2{,}031.7$ vs.\ noise prediction $2{,}032.1$ (the tall-matrix limit $d_{\mathrm{eff}} \to n$); expert down-projections $(524{,}288 \times 512)$: $511.5$ vs.\ $511.5$. Full-attention $\mathbf{W}_Q$ $(8{,}192 \times 2{,}048)$: $1{,}519$ vs.\ $1{,}638$. Linear-attention QKV projections $(8{,}192 \times 2{,}048)$: $1{,}565$ vs.\ $1{,}638$; linear-attention output projections $(2{,}048 \times 4{,}096)$: $1{,}026$ vs.\ $1{,}365$. Shared-expert and $\mathbf{W}_K/\mathbf{W}_V$ projections $(512 \times 2{,}048)$: $309$ to $374$ vs.\ $410$. MoE routers $(256 \times 2{,}048)$: $206$ vs.\ $228$. The strongest deviations from the noise ceiling are the $32$-row linear-attention $a/b$ projections ($d_{\mathrm{eff}} \approx 25$ vs.\ $31.5$) and a single vision-tower MLP output projection with $d_{\mathrm{eff}} = 83$ against a prediction of $909$: small or structurally atypical tensors, detectable at probe time via the multi-probe sample variance.

\paragraph{Bit-width invariance.}
Median $d_{\mathrm{eff}}$ at 2/3/4/8 bits: $386.1/386.0/386.1/386.2$ on the 35B model and $1{,}469/1{,}424/1{,}476/1{,}457$ on the 9B model. The spectral shape of RTN error is independent of its magnitude across a $64\times$ range of error power on the MoE model and varies by under $4\%$ on the dense model. Flatness is a property of the rounding mechanism, not of a particular precision. The MAE convergence slope is between $-0.49$ and $-0.51$ at every bit-width on both models.

\paragraph{Data.}
Raw per-tensor measurements (exact $\|\Delta\|_F^2$, $\|\Delta^\top\Delta\|_F^2$, $\|W\|_F^2$, 200 single-probe samples per tensor per bit-width) and the measurement script are archived at \texttt{RAM/results/deff/} in the project repository.

\section{Reproduction Instructions}
\label{app:reproduction}

Pre-quantized models are published at \url{https://huggingface.co/baa-ai}. The analysis and conversion scripts are \texttt{experiments/probe\_allocator\_v2.py}, \texttt{experiments/probe\_mla.py}, and \texttt{experiments/convert\_probe\_model.py} in the project repository; the main results reproduce with:

\begin{verbatim}
# Probe analysis (50 propagated probes, six bit-widths, group 64)
python experiments/probe_allocator_v2.py \
    --model /path/to/bf16-model --budget-gb 21 \
    --num-probes 50 --lazy --output manifest.json
# add --min-bits 4 for dense or <=128-expert models below ~30% of BF16

# Convert using the manifest
python experiments/convert_probe_model.py \
    --model /path/to/bf16-model --manifest manifest.json \
    --output /path/to/quantized

# Evaluate perplexity (mean and median)
python eval_perplexity.py --model /path/to/quantized \
    --num-sequences 128 --seq-length 2048 --seed 42
\end{verbatim}

\paragraph{Software versions.}
Python 3.12.0, mlx\_lm 0.30.4, PyTorch 2.6.0, SciPy 1.17.0, safetensors 0.6.2; lm-evaluation-harness 0.4.11 for MMLU.

\paragraph{Hardware.}
Apple Mac Studio, M2 Ultra, 192\,GB unified memory, macOS 15.5.

\end{document}